\pdfoutput=1
\documentclass{article}

\PassOptionsToPackage{numbers, sort, compress}{natbib}

 \usepackage[main, final]{neurips_2026}
\usepackage[utf8]{inputenc}%
\usepackage[T1]{fontenc}%
\usepackage{hyperref}%
\usepackage{url}%
\usepackage{booktabs}%
\usepackage{amsfonts}%
\usepackage{nicefrac}%
\usepackage{microtype}%
\usepackage{xcolor}%

\usepackage{nanostyle}

\usepackage{tikz}
\usetikzlibrary{
  arrows.meta,
  calc,
  fit,
  positioning,
  shapes.geometric
}

\title{GAE Falls Short in Imperfect-Information Self-Play Reinforcement Learning}

\author{%
  Zhiyuan Fan  \\
  MIT\\
  \texttt{fanzy@mit.edu} \\
  \And
  Gabriele Farina \\
  MIT \\
  \texttt{gfarina@mit.edu} 
}

\begin{document}

\maketitle

\begin{abstract}
Competitive multi-agent reinforcement learning in imperfect-information games requires agents to act under partial observability and against adversarial opponents, necessitating stochastic policies. While self-play reinforcement learning with Proximal Policy Optimization (PPO) has achieved strong empirical success, its standard advantage estimator, generalized advantage estimation, suffers from additional variance due to the sampling of stochastic future actions. This variance is amplified in equilibrium self-play because of the stochastic nature of the equilibrium policy and persists even when the critic is exact.
We address this bottleneck by introducing \textbf{$\boldsymbol{Q}$-boosting}, a variance-reduced advantage estimator based on a centralized action-value critic, and propose \textbf{Variance-Reduced Policy Optimization (VRPO)}, incorporating this new estimator. The algorithm replaces sampled multi-step backups with a multi-step Expected SARSA$(\lambda)$ trace, computing policy expectations at each step to average out action-sampling noise, while retaining PPO's clipped objective and on-policy actor updates. Empirically, VRPO achieves strong performance in various games: In mid-sized games where exploitability is computable, VRPO lowers the exact exploitability by more than $15\%$ compared to tuned baselines. In Dou Dizhu, it sets a new state of the art over PerfectDou. In Heads-Up No-Limit Texas Hold'em, it posts a positive win rate against Slumbot without any blueprint raise sizes, search, or subgame solving at test time.
\end{abstract}

\section{Introduction}

Multi-agent reinforcement learning (MARL) is a central challenge in reinforcement learning, where multiple agents interact competitively under imperfect information. During gameplay, players must take actions based only on the private information they observe. The outcome of the game is determined by both the hidden states and the actions of the opponents.

Learning policies with equilibrium-like robustness is a common goal in such games. This can be formulated as learning a policy that maximizes reward against the most adversarial opponent. However, such equilibrium concepts are generally intractable, and the best response is difficult to evaluate in large-scale games. Therefore, performance is often evaluated using heads-up play against a fixed opponent that is not encountered during training.

This problem is challenging for two main reasons. First, the agent needs to act under partial observability, which requires implicit reasoning about the hidden states of the opponents. Second, we want a policy that is robust to adaptive opponents, meaning that the training process needs to account for adversarial gameplay, even when such strategies may be far from the final policy. These challenges have motivated two broad paradigms: regret minimization and self-play RL.

The first line of work is based on \textbf{Counterfactual Regret Minimization (CFR)} \citep{zinkevich2007regret}. The algorithm provably converges to Nash equilibria in two-player zero-sum games. This approach is considered a leading method for solving extensive-form games such as No-Limit Texas Hold'em \citep{brown2017libratus}. Follow-up methods such as Deep CFR \citep{brown2019deep} and ReBeL \citep{brown2020combining} build on neural networks to approximate regrets and values, enabling the algorithms to scale to larger game trees and achieve strong performance in poker.

However, CFR-style methods can still struggle to scale in games such as Stratego or Dou Dizhu due to their large private-information spaces, long gameplay trajectories, and large action sets. These features make it difficult to expand the extensive-form game tree from any position, and thus regret minimization is considered both computationally and statistically inefficient in such settings.

The other, complementary line of work is based on \textbf{Self-Play Reinforcement Learning (Self-play RL)}, where policies are directly optimized using policy gradient methods. These algorithms have shown effectiveness in environments with vast private information, including Stratego \citep{perolat2022mastering,sokota2025superhuman} and Dou Dizhu \citep{zha2021douzero,yang2022perfectdou}. A recent line of work suggests that self-play can approximate equilibria in small games where the exact equilibrium can be directly evaluated, both theoretically~\citep{liu2022power,liu2024policy} and empirically~\citep{rudolph2025reevaluating}.
Among policy gradient methods, Proximal Policy Optimization (PPO) \citep{schulman2017proximal}, deployed under the centralized training with decentralized execution (CTDE) paradigm is a widely used and competitive pipeline, where decentralized actors are trained using centralized critics \citep{yu2022surprising}. 

However, even with its success, self-play RL still cannot achieve fully end-to-end training in games with a large dynamic range of payoffs, such as No-Limit Texas Hold'em. Existing work \citep{zhao2022alphaholdem} has achieved strong results in HUNL, but still relies on restricted betting-action representations to handle the large dynamic range of rewards.

In this paper, we identify a new bottleneck of self-play RL in imperfect-information games: \emph{advantage-estimation variance}. PPO-style algorithms use Generalized Advantage Estimation (GAE) \citep{schulman2015high} to aggregate multi-step temporal-difference residuals along sampled trajectories, thereby estimating the advantage with a variance--bias trade-off. GAE inherits randomness from future actions sampled in the multi-step trace. While in single-agent settings, policy randomness is largely exploratory and fades with training, it is intrinsic in self-play RL as equilibrium policies are generally stochastic. Different actions lead to different state values, so GAE produces high-variance gradient estimates.

We propose the \textbf{$\boldsymbol Q$-boosting} estimator to address this problem. This estimator reduces the variance of advantage estimation by removing intrinsic action-sampling noise, replacing GAE's sampled multi-step backup with a multi-step \emph{Expected SARSA}$(\lambda)$ trace computed from a centralized $Q$-critic. Instead of bootstrapping through sampled next actions, $Q$-boosting takes policy expectations at each backup step, averaging out noise from stochastic future actions.

We introduce \textbf{Variance-Reduced Policy Optimization (VRPO)} based on this new estimator. VRPO is a PPO-style algorithm with improved stability that achieves strong empirical performance, showing that reducing variance in advantage estimation yields substantial gains in equilibrium-quality self-play.

To summarize, our main contributions are:
\begin{itemize}[itemsep=0pt, topsep=0pt, leftmargin=*]
    \item We identify a previously unrecognized source of variance in GAE arising from sampled future actions in stochastic self-play. We show that this variance persists even with an exact state-value critic, and derive $Q$-boosting: a PPO-compatible advantage estimator that applies an Expected SARSA($\lambda$) trace to a centralized $Q$-critic to average out future action-sampling noise.
    \item We propose \textbf{Variance-Reduced Policy Optimization (VRPO)}, a PPO-style algorithm that lowers the exact exploitability by more than $15\%$ compared to tuned baselines (IPPO, MAPPO, MMD, and PPG) on mid-sized benchmarks. In Dou Dizhu, VRPO outperforms the previous \emph{state-of-the-art} agent, PerfectDou~\citep{yang2022perfectdou}, under a matched training budget. In heads-up no-limit Texas Hold'em, VRPO achieves positive performance against Slumbot, a strong poker baseline, without online search, subgame solving, or blueprint-based raise sizing.
\end{itemize}

\section{Preliminary}

\subsection{Multiplayer Imperfect Information Game}

We consider zero-sum multiplayer games with imperfect information, discrete time, and discrete action sets. Unlike cooperative MARL, we focus on competitive zero-sum settings, where the sum of players' returns is zero and equilibrium quality is measured by unilateral deviation gains.
This captures a broad class of competitive decision-making problems, including many board games.

Let $n$ denote the number of players, and let $\cS$ be the set of all distinct states. We model simultaneous steps, when present, as sequences of imperfect-information states while preserving the original information structure; players do not observe actions or hidden variables that would be unavailable in the original game. We also explicitly resolve all environmental randomness and model stochastic transitions (e.g., card draws) by introducing a special player, $\textsc{Nature}$, who follows a fixed policy. This modeling still allows for a random starting state, such as private hole cards in poker.

Let $\cE \subseteq \cS$ be the set of terminal states. Each non-terminal state $s \in \cS \setminus \cE$ in the game is assigned to a player $i(s) \in \llbracket n \rrbracket$ or to the \textsc{Nature} player. When the game transitions to such a state, the assigned player must choose an action $a$ from the set of available actions~$\cA$.\footnote{Formally, the available action set may depend on the state, i.e., $\cA = \cA(s)$. Since this dependence is immaterial to our analysis, we suppress it in the notation.} 
After player $i(s)$ takes action~$a$ in state~$s$, the game transitions to a new state~$s'$, and each player~$i$ receives a reward~$r_i(s, a)$. Since we introduce an environment player to handle stochasticity, the transition is considered \emph{deterministic}: $\PP(s' \mid s, a) = 1$ for the resulting state~$s'$ under the transition kernel~$\PP$.

Since the information in the game is imperfect, the active player~$i(s)$ cannot observe the underlying state~$s$ directly. Instead, the player perceives only the private information associated with the state, denoted by $o(s) \in \cO$. A player cannot distinguish between states that yield the same private information. Following standard terminology in the literature, the set of states that share the same private information is called an \emph{information set}, denoted by $\cI(i, o) \triangleq \{s \in \cS : i(s) = i,\, o(s) = o\}.$

A gameplay trajectory is represented as a sequence $h = (s_0, a_0, s_1, \dots, s_t)$, which starts from an initial state~$s_0$ and ends in a terminal state~$s_t \in \cE$, with $\PP(s_{\tau+1} \mid s_\tau, a_\tau) = 1$ for each timestep~$0 \leq \tau < t$. The return of player~$i$ is the cumulative reward with discount factor~$\gamma \in (0, 1]$:
\[
    R_i(h) \triangleq \sum_{\tau=0}^{t-1} \gamma^{\tau} \cdot r_i(s_\tau, a_\tau).
\]

In the broader class of games we primarily consider, each player's return is determined solely by the terminal state~$s_t$. To align with the general framework above, we assign the reward to the final action by setting $r_i(s_\tau, a_\tau) = 0$ for every timestep $\tau < t - 1$, and we use a discount factor of~$\gamma = 1$. In this case, the return simplifies to: $R_i(h) = r_i(s_{t-1}, a_{t-1}).$

A policy is a mapping from private information to a distribution over available actions, denoted by $\pi_i: \cO \to \Delta(\cA)$ for each player~$i \in \llbracket n \rrbracket$. In the RL setting, each policy is parameterized by a model parameter~$\theta_i$, and we write $\pi_i(\cdot \mid o; \theta_i) \in \Delta(\cA)$ to denote the action distribution of player~$i$ given private information~$o$.

Let $\vtheta \triangleq (\theta_1, \dots, \theta_n)$ denote the collection of model parameters. We define the joint policy profile~$\vpi: \cS \to \Delta(\cA)$ as $\vpi(\cdot \mid s; \vtheta) \triangleq \pi_{i(s)} (\cdot \mid o(s); \theta_{i(s)}), $
which represents the action distribution of the player acting at state~$s$. In particular, for states~$s$ assigned to the environment player, i.e., those where $i(s) = \textsc{Nature}$, the joint policy profile~$\vpi(\cdot \mid s; \vtheta)$ is fixed and independent of the model parameters~$\vtheta$, corresponding instead to the environment's predefined transition distribution.

\subsection{State Value Function and Action Value Function}

Given a joint policy~$\vpi$, we define two quantities. The $V$-value function, or state-value function, $V^{\vpi}_i : \cS \rightarrow \RR$, denotes the expected \emph{return-to-go} for player~$i$ at state~$s$, assuming all players follow the joint policy. The $Q$-value function, $Q^{\vpi}_i : \cS \times \cA \rightarrow \RR$, denotes the expected return-to-go when a particular action~$a$ is taken in state~$s$:
\begin{align*}
    V^{\vpi}_i(s) \triangleq \Expt_{h}\!\bigg[\sum_{\tau \geq k} \gamma^{\tau-k} \cdot r_i(s_\tau, a_\tau) \;\bigg|\; s_k = s \bigg],
    Q^{\vpi}_i(s, a) \triangleq \Expt_{h}\!\bigg[\sum_{\tau \geq k} \gamma^{\tau-k} \cdot r_i(s_\tau, a_\tau) \;\bigg|\; s_k = s,\, a_k = a \bigg],
\end{align*}
where the trajectory $h = (s_0, a_0, s_1, \dots, s_t)$ is generated according to the joint policy~$\vpi$, and we condition on $s = s_k$ for some timestep~$k$. Because stochasticity is modeled explicitly by the \textsc{Nature} player, the transition kernel $\PP$ is deterministic, and the $Q$-value function can be computed from the Bellman equation in the style of Expected SARSA~\citep{van2009theoretical} as
\begin{align} \label{eq:bellman-q-function}
    Q^{\vpi}_i(s, a) = r_i(s, a) + \gamma \sum_{a' \in \cA} \vpi(a' \mid s') \cdot Q^{\vpi}_i(s', a'),
\end{align}
where the next state $s'$ satisfies $\PP(s' \mid s, a) = 1$. The $V$-value function $V^{\vpi}_i(s)$ can then be computed by summing the terms $\vpi(a \mid s) \cdot Q^{\vpi}_i(s, a)$ over the action set $\cA$.
In this way, backpropagation through this Bellman equation can be performed deterministically in~$\cO(|\cA|)$ time.

\subsection{Game-Theoretic Equilibrium}

We measure equilibrium quality by the gain available from unilateral deviations. For a joint policy~$\vpi$, the maximal deviation gain of player~$i$ is $
\Delta_i(\vpi)\triangleq
\max_{\pi'_i}\Expt_{h\sim(\pi'_i,\vpi_{-i})}[R_i(h)]
-\Expt_{h\sim\vpi}[R_i(h)].$
We define exploitability as the average deviation gain across all players:
\[
\Expl(\vpi)\triangleq \frac{1}{n}\sum_{i=1}^n \Delta_i(\vpi).
\]
A profile is an $\eps$-Nash equilibrium~\citep{roughgarden2010algorithmic} if the deviation gain $\Delta_i(\vpi)\le \eps$ for all players~$i$ (hence $\Expl(\vpi)\le\eps$). Conversely, $\Expl(\vpi)\le\eps$ implies the worst-case deviation gain $\max_i \Delta_i(\vpi)\le n\eps$. In two-player zero-sum games, \citet{liu2022power, liu2024policy} show policy gradient with external regularization converges to an $\eps$-Nash equilibrium in time polynomial in the number of states, and thus guarantees $\Expl(\vpi)\le \eps$ under our definition.

\section{Variance Reduction by \texorpdfstring{$\boldsymbol{Q}$}{Q}-Boosting}

A general framework for multi-agent reinforcement learning (MARL) is centralized training with decentralized execution (CTDE), commonly implemented via a \emph{centralized critic} and decentralized actors~\citep{yu2022surprising}. The key asymmetry is that, during training, the critic is allowed to ``see more of the game'' than any individual actor: while an actor $\pi_i(a_i \mid o_i)$ must act based only on its private information~$o_i$, the critic can condition on centralized information such as the global state, the joint action, or other agents' observations. This additional context often makes the critic substantially easier to train and improves the fidelity of the resulting baseline or advantage estimates used in policy gradient updates.

This distinction is especially important in MARL because gradient estimates inherit additional randomness from other agents' sampled actions, as well as from the non-stationarity introduced when all agents update simultaneously. A centralized critic provides a natural way to average out some of this uncertainty by taking expectations over actions that are not controlled (or not observed) by player~$i$, thereby reducing the variance of the policy gradient estimator while keeping the actors decentralized at execution. In this work, we focus on the CTDE setup, such as MAPPO~\citep{yu2022surprising}, where a centralized critic is used.

We begin by examining the policy gradient of player~$i$. Let $J_i(\vtheta)$ be the expected cumulative reward object corresponding to the joint policy profile $\vpi(\cdot; \vtheta)$, where $\vtheta \triangleq (\theta_1, \dots, \theta_n)$ is the collection of models. According to~\citet{schulman2015high}, the policy gradient of player~$i$ can be computed as
\begin{align} \label{eq:policy-gradient}
    \nabla_{\theta_i} J_i(\vtheta) = \Expt_{h \sim \vpi} \!\bigg[\sum_{\tau: i(s_\tau ) = i} A_i^{\vpi}(s_\tau, a_\tau) \cdot \nabla_{\theta_i} \log \pi_i(a_\tau \mid o_\tau; \theta_i) \bigg],
\end{align}
where player~$i$ makes a decision at state~$s_\tau$ based on their private information $o_\tau \triangleq o_i(s_\tau)$, and the advantage function is defined as $A_i^{\vpi}(s_\tau, a_\tau) \triangleq Q_i^{\vpi}(s_\tau, a_\tau) - V_i^{\vpi}(s_\tau)$, with the $Q$-value function~$Q^{\vpi}_i$ and the $V$-value function~$V^{\vpi}_i$.
In policy gradient methods that estimate~$\nabla_{\theta_i} J_i(\vtheta)$, the gradient is typically computed by sampling states~$s_\tau$ from trajectories induced by the current policy and estimating the corresponding advantages~$A^{\vpi}_i(s_\tau, a_\tau)$.

\begin{figure*}[!t]
    \centering
    \scalebox{0.8}{
\begin{tikzpicture}[
    every node/.style={font=\small},
    state/.style={circle, draw, thick, minimum size=20pt},
    state1/.style={state, fill=fancyBlue!40},
    state2/.style={state, fill=fancyRed!40},
    state3/.style={state, fill=fancyColor!40},
    trans/.style={->, shorten >=1pt, >={Stealth[round]}, semithick},
    trans0/.style={trans, dashed},
    trans1/.style={trans, very thick},
    infoset/.style={draw, thick, rounded corners=13pt, inner sep=3pt},
]
    \begin{scope}
        \node [state1] at (0,     0   ) (r)  {$\boldsymbol{\emptyset}$};
        \node [state1] at (1.5 ,  1.2 ) (h)  {$\mathbf{h}$};
        \node [state1] at (1.5 , -1.2 ) (t)  {$\mathrm{t}$};
        \node          at (3.5 ,  1.9 ) (hh) {$+1$};
        \node          at (3.5 ,  0.5 ) (ht) {$\boldsymbol{-1}$};
        \node          at (3.5 , -0.5 ) (th) {$-1$};
        \node          at (3.5 , -1.9 ) (tt) {$+1$};
        
        \node [anchor=east] at ($(h)+(-14pt,0pt)$) {$V_1(\mathrm{h}) = -1$};
        
        \draw [trans1] (r) to node[pos=0.2, yshift= 10pt] {$\mathbf{h}$} (h);
        \draw [trans ] (r) to node[pos=0.2, yshift=-10pt] {$\mathrm{t}$} (t);
        \draw [trans0] (h) to node[pos=0.35, yshift= 7pt] {$\mathrm{h}$} (hh);
        \draw [trans1] (h) to node[pos=0.35, yshift=-7pt] {$\mathbf{t}$} (ht);
        \draw [trans ] (t) to node[pos=0.35, yshift= 7pt] {$\mathrm{h}$} (th);
        \draw [trans0] (t) to node[pos=0.35, yshift=-7pt] {$\mathrm{t}$} (tt);
        
        \node [align=center, anchor=south] at (1.5, 2.2) {
            \textbf{GAE for Deterministic Policy}
        };
        \node [align=center, anchor=north] at (1.5,-2.2) {
            \textbf{Ground Truth}: $A^{\vpi}_1(\emptyset,\mathrm{h}) = 0$\\ [2pt]
            \textbf{Under Sampled Trajectory $\mathbf{ht}$:}\\ [2pt]
            $\delta_{1,1}^{\vphantom+} = r_1(\emptyset, \mathrm{h}) + V_1(\mathrm{h}) - V(\emptyset) = 0$\\[2pt]
            $\delta_{1,2}^{\vphantom+} = r_1(\mathrm{h},\mathrm{t}) + V_1(\mathrm{ht}) - V(\mathrm{h}) = 0$\\[2pt]
            $\hat A^{\mathrm{GAE}}_1 = \delta_{1,1} + \delta_{1,2} = 0$\\[6pt]
            \textbf{Low Variance}
        };
    \end{scope}
    \begin{scope}[xshift=165pt]
        \node [state2] at (0,     0   ) (r)  {$\boldsymbol{\emptyset}$};
        \node [state2] at (1.5 ,  1.2 ) (h)  {$\mathbf{h}$};
        \node [state2] at (1.5 , -1.2 ) (t)  {$\mathrm{t}$};
        \node          at (3.5 ,  1.9 ) (hh) {$+1$};
        \node          at (3.5 ,  0.5 ) (ht) {$\boldsymbol{-1}$};
        \node          at (3.5 , -0.5 ) (th) {$-1$};
        \node          at (3.5 , -1.9 ) (tt) {$+1$};

        \node [anchor=east] at ($(h)+(-14pt,0pt)$) {$V_1(\mathrm{h}) = 0$};
        
        \draw [trans1] (r) to node[pos=0.2, yshift= 10pt] {$\mathbf{h}$} (h);
        \draw [trans ] (r) to node[pos=0.2, yshift=-10pt] {$\mathrm{t}$} (t);
        \draw [trans ] (h) to node[pos=0.35, yshift= 7pt] {$\mathrm{h}$} (hh);
        \draw [trans1] (h) to node[pos=0.35, yshift=-7pt] {$\mathbf{t}$} (ht);
        \draw [trans ] (t) to node[pos=0.35, yshift= 7pt] {$\mathrm{h}$} (th);
        \draw [trans ] (t) to node[pos=0.35, yshift=-7pt] {$\mathrm{t}$} (tt);
        \node [infoset, fit=(h)(t)] (Iht) {};
        \draw [thick] ($(h)+(36:20pt)$) arc[start angle=36, end angle=-36, radius=20pt];
        \draw [thick] ($(t)+(36:20pt)$) arc[start angle=36, end angle=-36, radius=20pt];
        
        \node at ($(h)+(50pt,0)$) {Mixed};
        \node at ($(t)+(50pt,0)$) {Mixed};

        \node [align=center, anchor=south] at (1.5, 2.2) {
            \textbf{GAE for Stochastic Policy}
        };
        \node [align=center, anchor=north] at (1.5,-2.2) {
            \textbf{Ground Truth}: $A^{\vpi}_1(\emptyset,\mathrm{h}) = 0$\\ [2pt]
            \textbf{Under Sampled Trajectory $\mathbf{ht}$:}\\ [2pt]
            $\delta_{1,1}^{\vphantom+} = r_1(\emptyset, \mathrm{h}) + V_1(\mathrm{h}) - V(\emptyset) = 0$\\[2pt]
            $\delta_{1,2}^{\vphantom+} = r_1(\mathrm{h},\mathrm{t}) + V_1(\mathrm{ht}) - V(\mathrm{h}) = \boldsymbol{-1}$\\[2pt]
            $\hat A^{\mathrm{GAE}}_1 = \delta_{1,1} + \delta_{1,2} = \boldsymbol{-1}$\\[6pt]
            \textbf{High Variance}
        };
    \end{scope}
    \begin{scope}[xshift=330pt]
        \node [state3] at (0,     0   ) (r)  {$\boldsymbol{\emptyset}$};
        \node [state3] at (1.5 ,  1.2 ) (h)  {$\mathbf{h}$};
        \node [state3] at (1.5 , -1.2 ) (t)  {$\mathrm{t}$};
        \node          at (3.5 ,  1.9 ) (hh) {$+1$};
        \node          at (3.5 ,  0.5 ) (ht) {$\boldsymbol{-1}$};
        \node          at (3.5 , -0.5 ) (th) {$-1$};
        \node          at (3.5 , -1.9 ) (tt) {$+1$};
        
        \node [anchor=east] at ($(h)+(-14pt,0pt)$) {$\boldsymbol{Q_1(\mathrm{h}, \mathrm{t}) = -1}$};
        
        \draw [trans1] (r) to node[pos=0.2, yshift= 10pt] {$\mathbf{h}$} (h);
        \draw [trans ] (r) to node[pos=0.2, yshift=-10pt] {$\mathrm{t}$} (t);
        \draw [trans ] (h) to node[pos=0.35, yshift= 7pt] {$\mathrm{h}$} (hh);
        \draw [trans1] (h) to node[pos=0.35, yshift=-7pt] {$\mathbf{t}$} (ht);
        \draw [trans ] (t) to node[pos=0.35, yshift= 7pt] {$\mathrm{h}$} (th);
        \draw [trans ] (t) to node[pos=0.35, yshift=-7pt] {$\mathrm{t}$} (tt);
        \node [infoset, fit=(h)(t)] (Iht) {};
        \draw [thick] ($(h)+(36:20pt)$) arc[start angle=36, end angle=-36, radius=20pt];
        \draw [thick] ($(t)+(36:20pt)$) arc[start angle=36, end angle=-36, radius=20pt];

        \node at ($(h)+(50pt,0)$) {Mixed};
        \node at ($(t)+(50pt,0)$) {Mixed};

        \node [align=center, anchor=south] at (1.5, 2.2) {
            \textbf{ $\boldsymbol Q$-boosting for Stochastic Policy}
        };
        \node [align=center, anchor=north] at (1.5,-2.2) {
            \textbf{Ground Truth}: $A^{\vpi}_1(\emptyset,\mathrm{h}) = 0$\\ [2pt]
            \textbf{Under Sampled Trajectory $\mathbf{ht}$:}\\ [2pt]
            $\delta_{1,1}^+ = r_1(\emptyset, \mathrm{h}) + V_1(\mathrm{h}) - \boldsymbol{Q_1(\emptyset, \mathrm{h})} = 0$\\[2pt]
            $\delta_{1,2}^+ = r_1(\mathrm{h},\mathrm{t}) + V_1(\mathrm{ht}) - \boldsymbol{Q_1(\mathrm{h}, \mathrm{t})} = 0$\\[2pt]
            $\hat A^{\mathrm{boost}}_1 = Q_1(\emptyset, \mathrm{h}) - V(\emptyset) + \delta_{1,1}^+ + \delta_{1,2}^+ = 0$\\[6pt]
            \textbf{Low Variance}
        };
    \end{scope}
\end{tikzpicture}
}
    \vspace{-12pt}
    \caption{
        We compare GAE and $Q$-boosting in the matching-pennies game, where the first player receives a reward of~$+1$ if their action matches the second player's and~$-1$ otherwise. With perfect information, the second player's policy is deterministic, and GAE exhibits low variance. However, under imperfect information, the equilibrium policy of the second player mixes uniformly between playing $\mathrm{h}$ and $\mathrm{t}$. The action value of taking any action at state $\mathrm{h}$ generally differs from the state value~$V(\mathrm{h})$. As a result, GAE inherits variance from sampling future actions at state $\mathrm{h}$, leading to high-variance advantage estimates even with an exact critic. In contrast, $Q$-boosting replaces the state value with the action value, eliminating variance due to action sampling and producing a low-variance estimator. The discount factor is $\gamma = 1$ and trace parameter is $\lambda = 1$ for clarity.
    }
    \label{fig:gae-vs-qboosting}
    \vspace{-9pt}
\end{figure*}
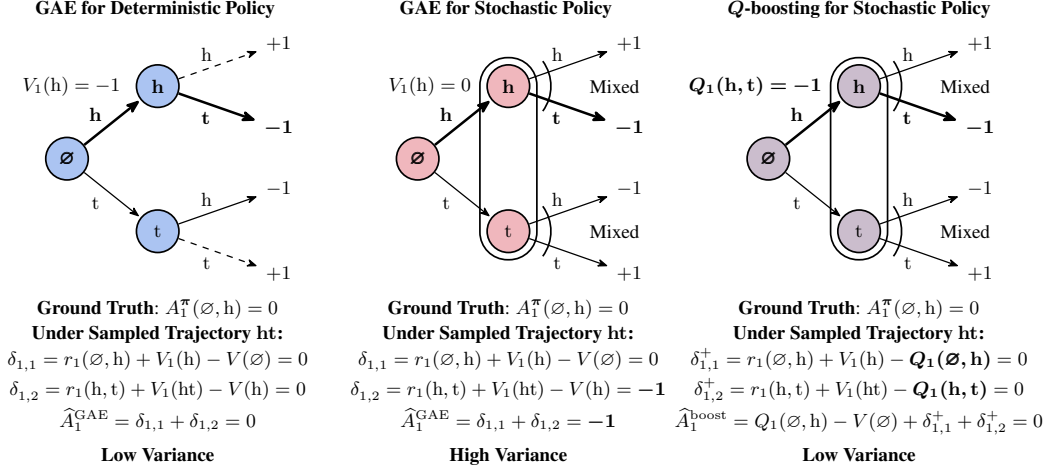

\subsection{The Intrinsic Variance of GAE}

An unbiased estimator of~$A^{\vpi}_i(s_\tau, a_\tau)$ can be derived from the return-to-go of the sampled trajectory. A more sophisticated method is the generalized advantage estimator (GAE) \citep{schulman2015high} used in PPO~\citep{schulman2017proximal}. Given a $V$-critic~$V_i(\cdot; \phi_i) : \cS \rightarrow \RR$ parameterized by~$\phi_i$, the temporal-difference (TD) residual is defined as $\delta_{i,\tau} \triangleq r_i(s_\tau, a_\tau) + \gamma V_i(s_{\tau+1}; \phi_i) - V_i(s_{\tau}; \phi_i).$
For a trace parameter~$\lambda \in [0, 1]$, the generalized advantage estimate is given by
\[
    \hat A^{\mathrm{GAE}}_{i,\tau} \triangleq \sum_{\tau' \geq \tau} (\lambda \gamma)^{\tau' - \tau} \cdot \delta_{i,\tau'}.
\]

Although GAE typically reduces variance relative to naïve Monte Carlo return-to-go estimation, it cannot eliminate the randomness arising from action sampling. 
The variance persists even when the $V$-critic is perfect, i.e., $V_i(\cdot; \phi_i) = V_i^{\vpi}$. In this case, the one-step TD residual can be computed as
$
    \delta_{i,\tau'} = r_i(s_{\tau'}, a_{\tau'}) + \gamma V^{\vpi}_i(s_{\tau'+1}) - V^{\vpi}_i(s_{\tau'}). 
$
For the primary term with $\tau' = \tau$, the expected contribution of the TD residual is 
$ 
    \Expt[\delta_{i,\tau} \mid s_{\tau}, a_\tau] = A^{\vpi}_i(s_\tau, a_\tau),
$
which provides an unbiased estimator of the advantage. For any auxiliary term with $\tau' > \tau$, however, the expected contribution is 
$
    \Expt[\delta_{i,\tau'} \mid s_{\tau}, a_\tau] = \Expt[A^{\vpi}_i(s_{\tau'}, a_{\tau'}) \mid s_{\tau}, a_\tau] = 0.
$
But the TD residuals~$\delta_{i,\tau'}$ are generally nonzero at equilibrium in imperfect-information games, which causes these auxiliary terms to be generally nonzero and incur nonzero variance. We illustrate this variance with a concrete example in Figure~\ref{fig:gae-vs-qboosting}.

Therefore, even with an exact value function, the one-step TD residual remains a non-degenerate random variable unless the policy is deterministic at~$s_{\tau'}$ or all actions have identical $Q$-values; these conditions generally do not hold in imperfect-information games. Since GAE aggregates these residuals across multiple time steps when~$\lambda > 0$, this action-sampling noise is propagated through the multi-step advantage estimator rather than eliminated.

\subsection{The \texorpdfstring{$\boldsymbol{Q}$}{Q}-Boosting Estimator}

We introduce \textbf{$\boldsymbol Q$-boosting} to address the issue mentioned above. The central idea is to replace the sampled multi-step TD trace used in GAE with a multi-step \emph{Expected SARSA} trace computed from a centralized $Q$-critic. The additional variance of GAE arises because the multi-step trace depends on sampled future actions whose action values may differ. $Q$-boosting reduces this source of noise by averaging over those future action distributions.

Specifically, given a $Q$-critic~$Q_i(\cdot; \phi_i): \cS \times \cA \to \RR$ and joint policy profile~$\vpi(\cdot; \vtheta)$, we define the state value as
\[
    V_{i}^{\vpi}(s; \phi_i) \triangleq \sum_{a \in \cA} \vpi(a \mid s; \vtheta) \cdot Q_i(s, a; \phi_i).
\]
With one-step Expected SARSA TD residual
$
    \delta^+_{i,\tau}
    \triangleq
    r_i(s_\tau, a_\tau)
    + \gamma V_{i}^{\vpi}(s_{\tau+1}; \phi_i)
    - Q_i(s_\tau, a_\tau; \phi_i)
$ and trace parameter $\lambda \in [0,1]$, $Q$-boosting estimates the advantage as
\begin{align} \label{eq:q-boosting-advantage}
    \hat A^{\mathrm{boost}}_{i,\tau}
    \triangleq
    Q_i(s_\tau, a_\tau; \phi_i) - V_{i}^{\vpi}(s_{\tau}; \phi_i) 
    + \sum_{\tau' \geq \tau} (\lambda \gamma)^{\tau' - \tau} \cdot \delta^+_{i,\tau'}.
\end{align}

The $Q$-boosting estimator also induces a multi-step Bellman-style update for the $Q$-critic by accumulating discounted Expected SARSA TD residuals along the trajectory, analogous to GAE:
\begin{align} \label{eq:q-boosting-q-target}
    Q^{\mathrm{target}}_i(s_{\tau}, a_{\tau}) \gets Q_i(s_\tau, a_{\tau}; \phi_i) + 
    \sum_{\tau' \geq \tau} (\lambda \gamma)^{\tau' - \tau} \cdot \delta^+_{i,\tau'}.
\end{align}
This target can be viewed as a multi-step generalization of the Expected SARSA update, analogous to the $\lambda$-return in GAE but applied to action-value functions.

Compared to GAE, the residual $\delta^+_{i,\tau'}$ uses the action-conditioned value $Q_i(s_{\tau'},a_{\tau'};\phi_i)$ rather than the state value $V_i(s_{\tau'};\phi_i)$. Consequently, when $Q_i=Q_i^{\vpi}$, each Expected-SARSA residual vanishes pathwise, so the multi-step trace no longer carries noise from sampled future actions. With an approximate $Q$-critic, $Q$-boosting replaces this action-sampling variance with critic error.

Formally, we have the following result. The theorem is proved in Appendix~\ref{app:proof-qboost-unbiased-pg}.
\begin{theorem}[Unbiasedness and variance reduction of Q-boosting]
\label{thm:qboost-unbiased-vr}
Fix a timestep $\tau$ at which player $i(s_\tau)=i$ acts, and suppose
trajectories are sampled on-policy from $\vpi$. With trace parameter
$\lambda=1$, the $Q$-boosting estimator based on any $Q$-critic
$Q_i(\cdot;\phi_i)$ gives an unbiased estimator of the score-function
policy-gradient term:
\[
    \Expt[
        \hat A^{\mathrm{boost}}_{i,\tau} \cdot
        \nabla_{\theta_i}\log \pi_i(a_\tau\mid o_\tau;\theta_i)
        \mid s_\tau
    ]
    =
    \Expt[
        A_i^{\vpi}(s_\tau,a_\tau) \cdot
        \nabla_{\theta_i}\log \pi_i(a_\tau\mid o_\tau;\theta_i)
        \mid s_\tau
    ].
\]
Moreover, if some future state $s_u$ reached with positive probability after $(s_\tau,a_\tau)$ exhibits non-zero variance in action-value, 
$\Var_{a\sim\vpi(\cdot\mid s_u)}[Q_i^{\vpi}(s_u,a)]>0$, for any $\lambda\in(0,1]$, there exists $\xi^\lambda_{i,\tau}>0$ such that any $Q$-critic satisfying
\(
    \|Q_i(\cdot;\phi_i)-Q_i^{\vpi}\|_\infty<\xi^\lambda_{i,\tau}
\)
yields smaller mean-squared advantage-estimation error than GAE with an exact $V$-critic and the same trace parameter $\lambda$:
\[
    \Expt\!\big[
        \big(
            \hat A^{\mathrm{boost}}_{i,\tau}
            -
            A_i^{\vpi}(s_\tau,a_\tau)
        \big)^2
        \,\big|\, s_\tau,a_\tau
    \big]
    <
    \Expt\!\big[
        \big(
            \hat A^{\mathrm{GAE}}_{i,\tau}
            -
            A_i^{\vpi}(s_\tau,a_\tau)
        \big)^2
        \,\big|\, s_\tau,a_\tau
    \big].
\]
\end{theorem}
This result shows that, for $\lambda=1$ and on-policy trajectories, the score-function policy-gradient term remains unbiased when the true advantage is replaced by the $Q$-boosting estimator, even with an arbitrary fixed, possibly inaccurate $Q$-critic. For $\lambda\in(0,1]$, whenever exact-$V$ GAE has nonzero future-action noise, a sufficiently accurate $Q$-critic yields lower conditional mean-squared advantage-estimation error than exact-$V$ GAE with the same trace parameter. In practice, we set $\lambda=0.95$ to obtain a good bias-variance tradeoff mirroring GAE~\citep{schulman2015high}.

We note that the state value $V_i^{\vpi}(s;\phi_i)$ is computed by enumerating the available actions in practice. For games with complex decisions, this can be done by decomposing complex decisions into multiple time steps such that the action set at each time step is small: In Dou Dizhu, this reduces the per-step maximum from $27{,}472$ actions~\citep{yang2022perfectdou} to $35$. See Appendix~\ref{sec:game-definition} for details.

\paragraph{Relation to prior expected-value estimators.}
We remark that similar estimators were proposed in the literature. The trace in $Q$-boosting follows the Expected-SARSA$(\lambda)$ form~\citep{van2009theoretical}, but is used here to construct PPO-style advantage estimates for stochastic self-play. Tree-Backup~\citep{precup2000eligibility} and Retrace~\citep{munos2016safe} provide value-evaluation traces for off-policy policy evaluation, COMA uses a one-step counterfactual baseline~\citep{foerster2018counterfactual}, and Q-Prop~\citep{gu2017qprop} uses an action-value critic as a control variate for single-agent policy gradients. However, our main contribution is to identify that GAE retains future-action sampling noise in self-play even with an exact $V$-critic, since equilibrium policies remain stochastic at convergence.

\section{Variance-Reduced Policy Optimization}

In this section, we introduce \textbf{Variance-Reduced Policy Optimization (VRPO)}, a variance-reduced variant of \emph{Proximal Policy Optimization} (PPO)~\citep{schulman2017proximal} that leverages the Q-boosting estimator~\eqref{eq:q-boosting-advantage}. The complete pseudo-code of the algorithm can be found in Appendix~\ref{sec:algo-vrpo}.

Compared to PPO, VRPO retains the same clipped policy gradient surrogate objective and trust-region-style updates, but modifies both the advantage estimation and critic learning procedures to reduce the variance of the policy gradient.

\subsection{Objectives and Variance-Reduced Estimation}

The core of VRPO replaces the GAE-based advantage used in PPO with the proposed $Q$-boosting estimator~\eqref{eq:q-boosting-advantage}, while retaining the clipped policy-gradient surrogate objective. Given a trajectory~$h = (s_0, a_0, s_1, \dots, s_t)$ sampled under the joint reference policy~$\vpi^{\mathrm{ref}} = (\pi^{\mathrm{ref}}_1, \dots, \pi^{\mathrm{ref}}_n)$, the surrogate loss for player~$i$ is defined as
\begin{align} \label{eq:loss-pg-term}
    \cL^{\mathrm{pg}}_i(h, \theta_i) \triangleq - \!\!\sum_{\tau: i(s_\tau) = i}\!\! 
    \min\!\bigg\{
    \frac{\pi_{i}(a_{\tau} \mid o(s_{\tau}); \theta_{i})}{\pi_i^{\mathrm{ref}}(a_{\tau} \mid o(s_{\tau}))} \hat A_{i,\tau}^{\mathrm{boost}}, 
    \mathrm{clamp}_{[1-\eps, 1+\eps]}\bigg(\frac{\pi_{i}(a_{\tau} \mid o(s_{\tau}); \theta_{i})}{\pi_i^{\mathrm{ref}}(a_{\tau} \mid o(s_{\tau}))}\bigg)\hat A_{i,\tau}^{\mathrm{boost}} \bigg\}. 
\end{align}

Following prior work in self-play RL~\citep{sokota2022unified,rudolph2025reevaluating}, we include an explicit policy-regularization term $\cL^{\mathrm{reg}}_i(h, \theta_i)$ to promote exploration and encourage last-iterate convergence. The policy is regularized using a KL-divergence penalty relative to a fixed reference distribution, which is chosen in practice to be uniform over the action space:
\begin{align} \label{eq:loss-reg-term}
    \cL^{\mathrm{reg}}_i(h, \theta_i) \triangleq \!\!\sum_{\tau: i(s_\tau) = i}\!\!  \mathrm{KL} \big(\pi_i(\cdot \mid o(s_{\tau}); \theta_i) \mmid \mathrm{Unif} \big) .
\end{align}

For the critic, VRPO replaces the single centralized $V$-critic in MAPPO~\citep{yu2022surprising} with a single centralized $Q$-critic with full-information inputs. The critic is trained by minimizing the squared regression loss
\begin{align} \label{eq:loss-val-term}
    \cL^{\mathrm{val}}_i(h, \phi_i) \triangleq \sum_{\tau \geq 0}
    \frac{1}{2} \big(Q_i(s_\tau, a_\tau; \phi_i) - Q^{\mathrm{target}}_i(s_{\tau}, a_{\tau})\big)^2,
\end{align}
where the target~$Q^{\mathrm{target}}_i$ is defined via the $Q$-boosting update in~\eqref{eq:q-boosting-q-target}.

\subsection{Training Procedure}

VRPO follows the standard PPO training schedule, alternating between on-policy actor updates and centralized critic regression. At the beginning of each iteration, the current policies are frozen as reference policies: $\pi^{\mathrm{ref}}_i \gets \pi_i(\cdot; \theta_i)$ for each player~$i \in \llbracket n \rrbracket$. Using the joint reference policy~$(\pi^{\mathrm{ref}}_1, \dots, \pi^{\mathrm{ref}}_n)$, the algorithm collects a batch of~$B$ rollout trajectories, denoted by~$\cD_{\mathrm{rollout}}$.

The actor parameters are optimized using only the on-policy rollout data. In each of the $K_{\mathrm{actor}}$ epochs, the rollout batch is partitioned into~$M$ minibatches, and each player updates its policy parameters by minimizing the PPO-style objective $ \cL^{\mathrm{pg}}_i(\cdot, \theta_i) + \alpha \cdot \cL^{\mathrm{reg}}_i(\cdot, \theta_i),$ where~$\alpha$ controls the strength of regularization, balancing the policy-gradient surrogate~\eqref{eq:loss-pg-term} and policy regularization~\eqref{eq:loss-reg-term}. This update follows the standard PPO minibatch training procedure. 

During actor optimization, VRPO recomputes the policy-expectation terms in
$\hat A^{\mathrm{boost}}_{i,\tau}$ using the current actor probabilities while keeping the critic values fixed, and uses the resulting advantages as ordinary stop-gradient PPO coefficients. This recomputation adds negligible overhead, since evaluating the PPO ratio already requires the current actor probabilities; the additional operation is only a policy-weighted sum over critic values.

The centralized $Q$-critics are trained for~$K_{\mathrm{critic}}$ epochs by minimizing the regression loss~$\cL^{\mathrm{val}}_i(\cdot, \phi_i)$, defined in~\eqref{eq:loss-val-term}.

\section{Evaluations}

In this section, we evaluate the empirical performance of VRPO across a range of imperfect-information games, spanning from mid-sized benchmarks, where exact exploitability can be computed, to large-scale domains where head-to-head performance is the standard measure. We provide a brief description of each game in Appendix~\ref{sec:game-description}, training details in Appendix~\ref{sec:training}, evaluation methodology in Appendix~\ref{sec:evaluation}, game-specific encodings in Appendix~\ref{sec:game-definition}, and training dynamics in Appendix~\ref{sec:dynamics}. Notably, we use the \emph{same} network architecture and hyperparameters across all games, adjusting only the number of training iterations and batch sizes to account for differences in game complexity.

\begin{figure*}[!t]
    \centering
    \includegraphics[scale=0.8]{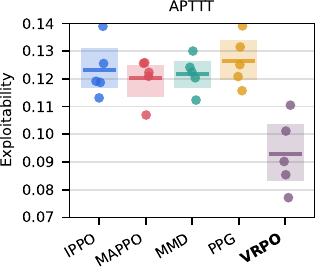}
    \hspace{8pt}
    \includegraphics[scale=0.8]{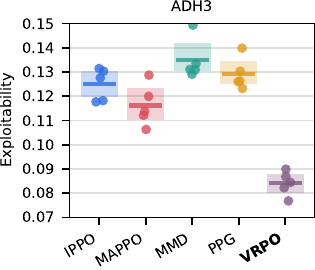}
    \hspace{8pt}
    \includegraphics[scale=0.8]{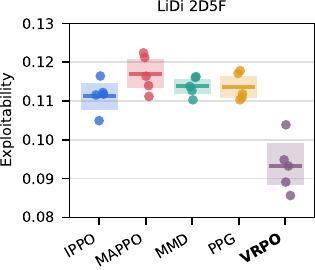}
    \caption{
        Exact exploitability (lower is better) of agents in various games under a shared training schedule. Each configuration is trained with $5$ different seeds.
    }
    \label{fig:exact-exploitability-various-algorithm}
    \vspace{-6pt}
\end{figure*}

\subsection{Mid-Sized Games with Exact Exploitability}

We benchmark VRPO on mid-sized games where exact exploitability can be computed: the number of information sets fits in memory, and full game-tree traversal is computationally feasible. These benchmarks include Abrupt Phantom Tic-Tac-Toe (APTTT), Abrupt Dark Hex 3 (ADH3), and Liar's Dice with 2 Dice and 5 Faces (LiDi 2D5F). The sizes of these games are listed in Table~\ref{tab:game-sizes}. We remark that these games are more than $4$ orders of magnitude larger than Leduc Hold'em~\citep{southey2012bayes}, a common benchmark for CFR-style algorithms.

\begin{wraptable}{r}{0.5\textwidth}
    \vspace{-9pt}
    \centering
    \small
    \setlength{\tabcolsep}{6pt}
    \renewcommand{\arraystretch}{1.15}
    \begin{tabular}{@{}lcc@{}}
    \toprule
    \textbf{Game} &  Infosets & Game States \\
    \midrule
    Leduc Hold'em & $936$ & $9{,}457$ \\
    APTTT & $23.3 \times 10^6$ & $27.1 \times 10^9$ \\
    ADH3  & $27.3 \times 10^6$ & $29.3 \times 10^9$ \\
    LiDi 2D5F & $52.4 \times 10^6$ & $1.31 \times 10^9$ \\
    \bottomrule
    \end{tabular}
    \caption{Game sizes in benchmarks}
    \label{tab:game-sizes}
\end{wraptable}

We compare VRPO with several policy-gradient-based baselines, including Independent PPO (IPPO)~\citep{schulman2017proximal,de2020independent}, MAPPO~\citep{yu2022surprising}, MMD~\citep{sokota2022unified}, and PPG~\citep{cobbe2021phasic}. All algorithms are trained for $2{,}000$ iterations with a batch size of $B = 2048$ trajectories. Learning rates and other optimization hyperparameters are matched across methods, while algorithm-specific hyperparameters follow the default settings provided by the CleanRL library~\citep{huang2022cleanrl}.
To ensure comparable compute budgets, we use $K_{\mathrm{actor}} = 4$ actor epochs and $K_{\mathrm{critic}} = 4$ critic epochs per iteration in VRPO, each with $M = 4$ minibatches. This mirrors the standard PPO training schedule of $K = 4$ epochs and $M = 4$ minibatches for both policy and value updates, as implemented in CleanRL. 

We report the results in Figure~\ref{fig:exact-exploitability-various-algorithm}. Baseline algorithms are tuned and outperform the results reported by \citet{rudolph2025reevaluating}. VRPO consistently achieves more than $15\%$ lower exploitability than all baselines across these games, with an additional computation cost of around $5\%$ (see Appendix~\ref{sec:training-time}). More hyperparameter-ablation experiments are provided in Appendices~\ref{sec:ablate-hyper} and~\ref{sec:ablate-vrpo}, including PPO clipping coefficient $\eps$, regularization coefficient $\alpha$, and trace parameter $\lambda$.

\subsection{Dou Dizhu}

\begin{wrapfigure}{r}{0.5\textwidth}
    \vspace{-12pt}
\centering
\includegraphics[scale=0.8]{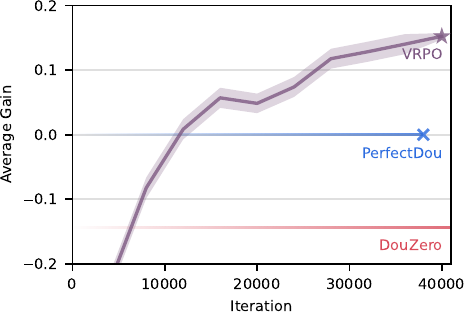}
\vspace{-9pt}
    \caption{
        Role-averaged gain (higher is better; values above zero indicate a win) against PerfectDou~\citep{yang2022perfectdou}, over the course of VRPO training. 
    }
    \vspace{-11pt}
\label{fig:doudizhu-training}
\end{wrapfigure}

We apply VRPO to Dou Dizhu, a three-player game between one Landlord and two independent Peasants. The game comprises at least $10^{53}$ information sets, each averaging a size of $10^{23}$~\citep{zha2019rlcard, zha2021douzero}. Dou Dizhu is a standard benchmark for self-play RL, with prior agents including DeltaDou~\citep{jiang2019deltadou}, DouZero~\citep{zha2021douzero}, and PerfectDou~\citep{yang2022perfectdou}, and implementations in RLCard~\citep{zha2019rlcard} and OpenSpiel~\citep{lanctot2019openspiel}.

We train the agent for $40{,}000$ iterations with a batch size of $B=2048$ trajectories, which takes approximately $55$~hours on $4{\times}$ RTX~5090 GPUs. This corresponds to approximately $2.65 \times 10^9$ timesteps, closely matching the $2.5 \times 10^9$ timesteps of PerfectDou~\citep{yang2022perfectdou}.

We evaluate the role-averaged performance of the VRPO agent against PerfectDou~\citep{yang2022perfectdou}, the previous state-of-the-art trained by MAPPO, in a setting where one agent plays the Landlord and the other controls the two independent Peasants. Figure~\ref{fig:doudizhu-training} shows that VRPO surpasses PerfectDou after only about $40\%$ as many training timesteps as PerfectDou, despite relying solely on terminal rewards rather than a manually designed loss function. Each evaluation is averaged over $10{,}000$ duplicate deals, except for the final policy, which is evaluated over $100{,}000$ duplicate deals. We include DouZero~\citep{zha2021douzero} as an additional reference point against PerfectDou; DouZero was trained for approximately $10^{10}$ timesteps, corresponding to more than $150{,}000$ iterations under our training schedule.

We also evaluate the final VRPO checkpoint against prior baselines in Table~\ref{tab:doudizhu-avg}, with each matchup evaluated over $100{,}000$ duplicate deals. The reported uncertainty is $1.0 \times \mathrm{SEM}$. We further report position-specific performance in Appendix~\ref{sec:more-doudizhu-figures}. VRPO achieves consistently strong performance and outperforms PerfectDou and DouZero across all matchups, showing that VRPO scales effectively to large-scale imperfect-information games and establishes a new state of the art.

\begin{table}[!h]
    \vspace{-3pt}
    \centering
    \small
    \setlength{\tabcolsep}{6pt}
    \renewcommand{\arraystretch}{1.15}
    \begin{tabular}{@{}lccc@{}}
    \toprule
    \textbf{Hero $\backslash$ Villain} & \textbf{VRPO} & PerfectDou & DouZero \\
    \midrule
    \textbf{VRPO (ours)} & --- & $0.152 \pm 0.004$ & $0.286 \pm 0.004$ \\
    PerfectDou & $-0.152  \pm 0.004$ & --- & $0.141 \pm 0.004$ \\
    DouZero  & $-0.286  \pm 0.004$ & $-0.141 \pm 0.004$ & --- \\
    \bottomrule
    \end{tabular}
    \vspace{6pt}
    \caption{
        Average gain of the row agent playing against the column player.
    }
    \label{tab:doudizhu-avg}
    \vspace{-12pt}
\end{table}

\subsection{Heads-Up No-Limit Texas Hold'em}

\begin{wrapfigure}[14]{r}{0.5\textwidth}
    \vspace{-9pt}
    \centering
    \includegraphics[scale=0.8]{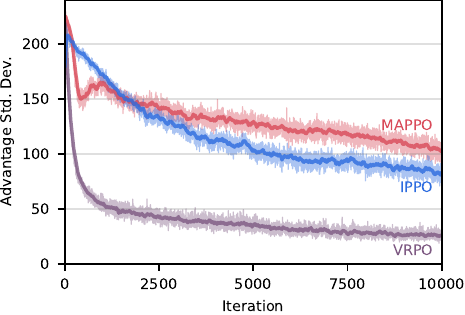}
    \vspace{-6pt}
    \caption{
        Comparison of the standard deviation of the used advantage~$\hat A$ during the first $10{,}000$ steps of training under different methods.
    }
    \vspace{-6pt}
    \label{fig:advantage-variance}
\end{wrapfigure}

We apply VRPO to Heads-Up No-Limit Texas Hold'em with an initial stack of $200$ big blinds. The agent is trained for $40{,}000$ iterations with a batch size of $B = 8192$ trajectories, totaling approximately $1.51 \times 10^9$ timesteps. Training takes roughly $63$~hours on $4{\times}$ RTX~5090 GPUs.

We compare training with the IPPO and MAPPO baselines. As shown in Figure~\ref{fig:advantage-variance}, VRPO significantly reduces the variance of the advantage compared to the baseline algorithms.

We evaluate the performance of the VRPO agent using the 2022 version of Slumbot~\citep{ericgjackson_slumbot2019}. Over a session of $1{,}000{,}000$ hands, the VRPO agent achieves an average gain of $33 \pm 19$\,mBB/hand. For comparison, an always-folding bot loses $750$\,mBB/hand. While this does not constitute a decisive win, the result is achieved without any form of online search, subgame solving, or blueprint-based raise amounts, demonstrating the scalability of VRPO in imperfect-information domains.

\section{Conclusion}

In this paper, we identify that sampling future actions in GAE introduces additional variance, which is especially problematic in equilibrium learning, where policies are typically stochastic. We address this with \emph{$Q$-boosting}, which replaces action-sampled multi-step backups with a multi-step Expected SARSA$(\lambda)$ trace. Incorporating this new estimator into PPO yields VRPO. Empirically, VRPO achieves strong performance across a range of games at different scales, including a new state of the art in Dou Dizhu.

\paragraph{Limitations and Scope.} VRPO targets discrete-action games, where the policy expectation can be computed exactly. With task-specific decomposition, this remains practical at scale, with at most $35$ actions per step in Dou Dizhu and $16$ in HUNL. Continuous actions are a natural extension, where the exact sum can be replaced by quadrature, sampling, or a DDPG-style deterministic surrogate. As for PPO and MAPPO, no general convergence theorem is available for VRPO under nonlinear function approximation. Our focus is stochastic self-play with centralized critics.

\begin{ack}
This work was supported in part by National Science Foundation awards CCF-2443068 and IIS-2552046, Office of Naval Research grant N000142512296, and an AI2050 Early Career Fellowship.
\end{ack}

\bibliographystyle{unsrtnat}
\bibliography{refs}

\clearpage
\appendix

{\LARGE Appendix}

\section{Additional Related Work}

One of the major breakthroughs in large-scale equilibrium computation for imperfect-information games was Counterfactual Regret Minimization (CFR)~\citep{zinkevich2007regret} and its variants. In particular, CFR+ enabled Cepheus~\citep{tammelin2015solving} to solve Heads-Up Limit Texas Hold'em. In the No-Limit variant, successor methods such as Libratus~\citep{brown2017libratus}, which introduced a search-based approach, and DeepStack~\citep{moravvcik2017deepstack}, which employed a learning-based approach, defeated professional players. ReBeL~\citep{brown2020combining} later unified these approaches to achieve state-of-the-art performance.

While regret minimization methods have seen significant success in poker, policy optimization approaches have also made substantial progress in other imperfect-information games. Early foundations trace back to Fictitious Play~\citep{brown1951iterative}, while modern approaches typically rely on self-play reinforcement learning. Landmark achievements include OpenAI Five~\citep{berner2019dota}, defeating the world champion Dota 2 team OG using large-scale self-play with PPO~\citep{schulman2017proximal}, and AlphaStar~\citep{vinyals2019grandmaster}, reaching Grandmaster level in StarCraft II using population-based and self-play methods related to PSRO~\citep{lanctot2017unified}.

In board games, AlphaGo and its successor AlphaZero~\citep{silver2017mastering, silver2018general} mastered perfect-information games such as chess, shogi, and Go using Monte Carlo Tree Search. For imperfect-information games, DeepNash~\citep{perolat2022mastering} reached human-expert-level performance in Stratego using Regularized Nash Dynamics. More recently, Ataraxos~\citep{sokota2025superhuman} combined self-play RL with test-time search to achieve vastly superhuman performance in the game of Stratego.

\section{Variance-Reduced Policy Optimization} \label{sec:algo-vrpo}

\begin{algorithm}[!h]
\caption{Variance-Reduced Policy Optimization}
\label{algo:vrpo}
\begin{algorithmic}[1]
\Require Decentralized actors $\{\pi_i(\cdot;\theta_i)\}_{i=1}^n$ 
\Require Centralized critics $\{Q_i(\cdot;\phi_i)\}_{i=1}^n$
\Require Batch size $B$, minibatches $M$
\Require Actor epochs $K_{\mathrm{actor}}$, critic epochs $K_{\mathrm{critic}}$, 
\Require Regularization coefficient $\alpha > 0$
\For{iteration $=1,2,\dots$}
    \State Set reference policies $\pi_i^{\mathrm{ref}} \gets \pi_i(\cdot;\theta_i)$ for all $i$
    \State Collect $B$ trajectories $\cD_{\mathrm{rollout}}$ using joint $\vpi^{\mathrm{ref}}$

    \For{$K_{\mathrm{actor}}$ epochs $\times$ $M$ minibatches}
        \State Sample $|\cB| = B/M$ trajectories $\cB \subset \cD_{\mathrm{rollout}}$
        \State Update each $\theta_i$ by descending \Comment{\textbf{Actor update}}
        \[
        \frac{1}{|\cB|}\sum_{h \in \cB}
        \Big(\cL^{\mathrm{pg}}_i(h,\theta_i)
        + \alpha \cdot \cL^{\mathrm{reg}}_i(h,\theta_i)\Big)
        \]
    \EndFor
    
    \State Add $\cD_{\mathrm{rollout}}$ to cyclic replay buffer $\cD_{\mathrm{replay}}$
    \State Set $\cB \leftarrow \cD_{\mathrm{rollout}}$
    \For{$K_{\mathrm{critic}}$ epochs $\times$ $M$ minibatches}
        \State Update each $\phi_i$ by descending \Comment{\textbf{Critic update}}
        \[
        \frac{1}{|\cB|}\sum_{h \in \cB}
        \cL^{\mathrm{val}}_i(h,\phi_i)
        \]
        \State Resample $|\cB| = B/M$ trajectories $\cB \subset \cD_{\mathrm{replay}}$
    \EndFor
\EndFor
\end{algorithmic}
\end{algorithm}

Within the algorithm, the policy gradient term $\cL^{\mathrm{pg}}_i(h,\theta_i)$ is given by \eqref{eq:loss-pg-term}, the regularization penalty term $\cL^{\mathrm{reg}}_i(h,\theta_i)$ by \eqref{eq:loss-reg-term}, and the value loss term $\cL^{\mathrm{val}}_i(h,\phi_i)$ by \eqref{eq:loss-val-term}.

\clearpage
\section{Proof of Theorem~\ref{thm:qboost-unbiased-vr}}
\label{app:proof-qboost-unbiased-pg}

\begin{proof}
We consider a finite-horizon trajectory
$h=(s_0,a_0,s_1,\ldots,s_T)$ sampled on-policy from the joint policy
profile $\vpi$. If a trajectory terminates early, we pad it with an
absorbing terminal state with zero reward and zero value. Hence the
horizon can be treated as fixed. Since all environment randomness is
represented by the \textsc{Nature} player, the transition after
conditioning on the current state and action is deterministic.

For compactness, write
\[
    \bar Q_i(s,a) \triangleq Q_i(s,a;\phi_i),
    \qquad
    \bar V_i(s) \triangleq
    \sum_{a\in\cA(s)} \vpi(a\mid s)\bar Q_i(s,a).
\]
The true state value satisfies
\[
    V_i^{\vpi}(s)
    =
    \sum_{a\in\cA(s)} \vpi(a\mid s) \cdot Q_i^{\vpi}(s,a).
\]
At terminal states, we set
$\bar V_i(s_T)=V_i^{\vpi}(s_T)=0$.

\paragraph{Unbiasedness of the score-function term.}
We first prove the unbiasedness statement for $\lambda=1$. In this case,
the $Q$-boosting estimator is
\[
\begin{aligned}
    \hat A^{\mathrm{boost}}_{i,\tau}
    &=
    \bar Q_i(s_\tau,a_\tau)-\bar V_i(s_\tau)
    +
    \sum_{u=\tau}^{T-1}
    \gamma^{u-\tau}
    \big(
        r_i(s_u,a_u)
        +
        \gamma \bar V_i(s_{u+1})
        -
        \bar Q_i(s_u,a_u)
    \big).
\end{aligned}
\]
The initial $\bar Q_i(s_\tau,a_\tau)$ term cancels with the
$-\bar Q_i(s_\tau,a_\tau)$ term in the residual at $u=\tau$, so
\[
\begin{aligned}
    \hat A^{\mathrm{boost}}_{i,\tau}
    &=
    \sum_{u=\tau}^{T-1}
        \gamma^{u-\tau} r_i(s_u,a_u)
    - \bar V_i(s_\tau)
    +
    \sum_{u=\tau}^{T-1}
        \gamma^{u-\tau+1}\bar V_i(s_{u+1})  
    -
    \sum_{u=\tau+1}^{T-1}
        \gamma^{u-\tau}\bar Q_i(s_u,a_u).
\end{aligned}
\]
For any future nonterminal state $s_u$ with $u>\tau$,
\[
    \Expt[
        \bar Q_i(s_u,a_u)
        \mid s_u
    ]
    =
    \sum_{a\in\cA(s_u)}
        \vpi(a\mid s_u) \cdot \bar Q_i(s_u,a)
    =
    \bar V_i(s_u).
\]
Therefore, by the tower property, the future $\bar Q_i$ terms become
$\bar V_i$ terms in conditional expectation and telescope:
\[
\begin{aligned}
    \Expt[
        \hat A^{\mathrm{boost}}_{i,\tau}
        \mid s_\tau,a_\tau]
    &=
    \Expt\!\bigg[
        \sum_{u=\tau}^{T-1}
        \gamma^{u-\tau} r_i(s_u,a_u)
        \,\bigg|\, s_\tau,a_\tau
    \bigg]
    - \bar V_i(s_\tau) =
    Q_i^{\vpi}(s_\tau,a_\tau)-\bar V_i(s_\tau).
\end{aligned}
\]
Using the tower property again,
\[
\begin{aligned}
    \Expt[
        \hat A^{\mathrm{boost}}_{i,\tau}
        \cdot \nabla_{\theta_i}\log \pi_i(a_\tau\mid o_\tau;\theta_i)
        \mid s_\tau
    ]
    &=
    \Expt\!\big[
        \big(
            Q_i^{\vpi}(s_\tau,a_\tau)-\bar V_i(s_\tau)
        \big)
        \cdot \nabla_{\theta_i}\log \pi_i(a_\tau\mid o_\tau;\theta_i)
        \mid s_\tau
    \big].
\end{aligned}
\]
The second term is a state-dependent baseline. Since $i(s_\tau)=i$,
conditioning on $s_\tau$ fixes $o_\tau$, 
\[
\begin{aligned}
    \Expt[
        \bar V_i(s_\tau) \cdot 
        \nabla_{\theta_i}\log \pi_i(a_\tau\mid o_\tau;\theta_i)
        \mid s_\tau
    ]
    &=
    \bar V_i(s_\tau)
    \sum_{a\in\cA(s_\tau)}
        \pi_i(a\mid o_\tau;\theta_i) \cdot 
        \nabla_{\theta_i} \log \pi_i(a\mid o_\tau;\theta_i) \\
    &=
    \bar V_i(s_\tau) \cdot 
    \nabla_{\theta_i}
    \sum_{a\in\cA(s_\tau)}
        \pi_i(a\mid o_\tau;\theta_i) \\
    &=0.
\end{aligned}
\]
Thus
\[
    \Expt[
        \hat A^{\mathrm{boost}}_{i,\tau} \cdot 
        \nabla_{\theta_i}\log \pi_i(a_\tau\mid o_\tau;\theta_i)
        \mid s_\tau
    ]
    =
    \Expt[
        Q_i^{\vpi}(s_\tau,a_\tau) \cdot 
        \nabla_{\theta_i}\log \pi_i(a_\tau\mid o_\tau;\theta_i)
        \mid s_\tau
    ].
\]
Similarly, $V_i^{\vpi}(s_\tau)$ is also a state-dependent baseline, so
\[
\begin{aligned}
    \Expt[
        A_i^{\vpi}(s_\tau,a_\tau) \cdot 
        \nabla_{\theta_i}\log \pi_i(a_\tau\mid o_\tau;\theta_i)
        \mid s_\tau
    ]
    &=
    \Expt\!\big[
        \big(
            Q_i^{\vpi}(s_\tau,a_\tau)
            -
            V_i^{\vpi}(s_\tau)
        \big) \cdot 
        \nabla_{\theta_i}\log \pi_i(a_\tau\mid o_\tau;\theta_i)
        \,\big|\, s_\tau
    \big] \\
    &=
    \Expt\!\big[
        Q_i^{\vpi}(s_\tau,a_\tau) \cdot 
        \nabla_{\theta_i}\log \pi_i(a_\tau\mid o_\tau;\theta_i)
        \,\big|\, s_\tau
    \big].
\end{aligned}
\]
Combining the two displays proves
\[
    \Expt[
        \hat A^{\mathrm{boost}}_{i,\tau} \cdot 
        \nabla_{\theta_i}\log \pi_i(a_\tau\mid o_\tau;\theta_i)
        \mid s_\tau
    ]
    =
    \Expt[
        A_i^{\vpi}(s_\tau,a_\tau) \cdot 
        \nabla_{\theta_i}\log \pi_i(a_\tau\mid o_\tau;\theta_i)
        \mid s_\tau
    ].
\]

\paragraph{Mean-squared error of exact-$V$ GAE.}
We next prove the variance-reduction statement. Fix trace parameter
$\lambda\in(0,1]$ and define $\rho\triangleq\lambda\gamma$.
Consider GAE with the exact value critic
$V_i(\cdot)=V_i^{\vpi}(\cdot)$. Its one-step TD residual is
\[
    \delta_{i,u}
    =
    r_i(s_u,a_u)
    +
    \gamma V_i^{\vpi}(s_{u+1})
    -
    V_i^{\vpi}(s_u).
\]
By the Bellman equation for $Q_i^{\vpi}$ and the deterministic-transition
model,
\[
    Q_i^{\vpi}(s_u,a_u)
    =
    r_i(s_u,a_u)
    +
    \gamma V_i^{\vpi}(s_{u+1}).
\]
Hence
\[
    \delta_{i,u}
    =
    Q_i^{\vpi}(s_u,a_u)
    -
    V_i^{\vpi}(s_u)
    =
    A_i^{\vpi}(s_u,a_u).
\]
Therefore,
\[
\begin{aligned}
    \hat A^{\mathrm{GAE}}_{i,\tau}
    -
    A_i^{\vpi}(s_\tau,a_\tau)
    &=
    \sum_{u=\tau}^{T-1}
        \rho^{u-\tau}
        A_i^{\vpi}(s_u,a_u)
    -
    A_i^{\vpi}(s_\tau,a_\tau) =
    \sum_{u=\tau+1}^{T-1}
        \rho^{u-\tau}
        A_i^{\vpi}(s_u,a_u).
\end{aligned}
\]
Define the conditional future-action noise level
\[
    \Gamma^\lambda_{i,\tau}
    \triangleq
    \Expt\!\bigg[
        \bigg(
            \sum_{u=\tau+1}^{T-1}
            \rho^{u-\tau}
            A_i^{\vpi}(s_u,a_u)
        \bigg)^2
        \,\bigg|\, s_\tau,a_\tau
    \bigg].
\]
Then
\[
    \Gamma^\lambda_{i,\tau}
    =
    \Expt\!\big[
        \big(
            \hat A^{\mathrm{GAE}}_{i,\tau}
            -
            A_i^{\vpi}(s_\tau,a_\tau)
        \big)^2
        \,\big|\, s_\tau,a_\tau
    \big].
\]

We now show that the non-value-equivalence condition implies
$\Gamma^\lambda_{i,\tau}>0$. We interpret this condition as follows:
after $(s_\tau,a_\tau)$, there exists a future timestep $u>\tau$ and a
nonterminal state $s$ reached with positive probability under $\vpi$ such
that
\[
    \Pr_{\vpi}(s_u=s \,\big|\, s_\tau,a_\tau)>0,
    \qquad
    \Var_{a\sim\vpi(\cdot\mid s)}
    [
        Q_i^{\vpi}(s,a)
    ]
    >0.
\]
Equivalently, two actions in the support of $\vpi(\cdot\mid s)$ have
different $Q_i^{\vpi}$-values.

Let $\cF_u$ denote the trajectory history up to state $s_u$,
before sampling $a_u$. For every $u>\tau$,
\[
\begin{aligned}
    \Expt[
        A_i^{\vpi}(s_u,a_u)
        \mid \cF_u
    ]
    =
    \sum_{a\in\cA(s_u)}
        \vpi(a\mid s_u) \cdot 
        \big(
            Q_i^{\vpi}(s_u,a)
            -
            V_i^{\vpi}(s_u)
        \big)=0.
\end{aligned}
\]
Thus the future advantage terms form a martingale-difference sequence.
For $u<v$,
\[
\begin{aligned}
    \Expt[
        A_i^{\vpi}(s_u,a_u) \cdot 
        A_i^{\vpi}(s_v,a_v)
        \
        \mid s_\tau,a_\tau
    ]
    &=
    \Expt\!\big[
        A_i^{\vpi}(s_u,a_u)
        \Expt[
            A_i^{\vpi}(s_v,a_v)
            \mid \cF_v
        ]
        \,\big|\, s_\tau,a_\tau
    \big] =0.
\end{aligned}
\]
Therefore the cross terms vanish, and
\[
    \Gamma^\lambda_{i,\tau}
    =
    \sum_{u=\tau+1}^{T-1}
        \rho^{2(u-\tau)}
        \Expt\!\big[
            \big(
                A_i^{\vpi}(s_u,a_u)
            \big)^2
            \,\big|\, s_\tau,a_\tau
        \big].
\]
At any state $s$ satisfying the non-value-equivalence condition,
\[
\begin{aligned}
    \Expt\!\big[
        \big(
            A_i^{\vpi}(s_u,a_u)
        \big)^2
        \,\big|\, s_u=s
    \big]
    &=
    \Var_{a\sim\vpi(\cdot\mid s)}
    [
        Q_i^{\vpi}(s,a)
    ]
    >0.
\end{aligned}
\]
Since such a state is reached with positive probability and
$\rho=\lambda\gamma>0$, at least one term in the above sum is strictly
positive. Hence $\Gamma^\lambda_{i,\tau}>0$.

\paragraph{Mean-squared error of $Q$-boosting.}
We now bound the error of the $Q$-boosting estimator with an approximate
$Q$-critic. Define
\[
    e_Q(s,a)
    \triangleq
    \bar Q_i(s,a)-Q_i^{\vpi}(s,a),
    \qquad
    e_V(s)
    \triangleq
    \bar V_i(s)-V_i^{\vpi}(s).
\]
Since $\bar V_i$ and $V_i^{\vpi}$ are policy averages of $\bar Q_i$ and
$Q_i^{\vpi}$, respectively,
\[
\begin{aligned}
    |e_V(s)|
    &=
    \bigg|
        \sum_{a\in\cA(s)}
            \vpi(a\mid s) \cdot e_Q(s,a)
    \bigg|  \le
    \| \bar Q_i-Q_i^{\vpi} \|_\infty.
\end{aligned}
\]
Let
\[
    \epsilon_Q
    \triangleq
    \| \bar Q_i-Q_i^{\vpi} \|_\infty.
\]
Using the Bellman identity
\[
    Q_i^{\vpi}(s_u,a_u)
    =
    r_i(s_u,a_u)
    +
    \gamma V_i^{\vpi}(s_{u+1}),
\]
the Expected-SARSA residual in $Q$-boosting satisfies
\[
\begin{aligned}
    \delta^+_{i,u}
    &=
    r_i(s_u,a_u)
    +
    \gamma\bar V_i(s_{u+1})
    -
    \bar Q_i(s_u,a_u) =
    \gamma e_V(s_{u+1})
    -
    e_Q(s_u,a_u).
\end{aligned}
\]
Subtracting the true advantage from the $Q$-boosting estimator gives
\[
\begin{aligned}
    \hat A^{\mathrm{boost}}_{i,\tau}
    -
    A_i^{\vpi}(s_\tau,a_\tau)
    &=
    e_Q(s_\tau,a_\tau)
    -
    e_V(s_\tau)
    +
    \sum_{u=\tau}^{T-1}
        \rho^{u-\tau}
        \big(
            \gamma e_V(s_{u+1})
            -
            e_Q(s_u,a_u)
        \big).
\end{aligned}
\]
The $e_Q(s_\tau,a_\tau)$ term cancels with the
$-e_Q(s_\tau,a_\tau)$ term inside the residual at $u=\tau$, so
\[
\begin{aligned}
    \hat A^{\mathrm{boost}}_{i,\tau}
    -
    A_i^{\vpi}(s_\tau,a_\tau)
    &=
    -e_V(s_\tau)
    +
    \gamma e_V(s_{\tau+1}) +
    \sum_{u=\tau+1}^{T-1}
        \rho^{u-\tau}
        \big(
            \gamma e_V(s_{u+1})
            -
            e_Q(s_u,a_u)
        \big).
\end{aligned}
\]
Let $L_\tau\triangleq T-\tau$ be the remaining number of action steps.
Using $|e_Q(s,a)|\le\epsilon_Q$ and $|e_V(s)|\le\epsilon_Q$, we obtain
the pathwise bound
\[
\begin{aligned}
    \big|
        \hat A^{\mathrm{boost}}_{i,\tau}
        -
        A_i^{\vpi}(s_\tau,a_\tau)
    \big|
    &\le
    (1+\gamma)\epsilon_Q
    +
    \sum_{u=\tau+1}^{T-1}
        \rho^{u-\tau}
        (1+\gamma)\epsilon_Q =
    (1+\gamma)\epsilon_Q
    \sum_{k=0}^{L_\tau-1}
        \rho^k .
\end{aligned}
\]
Consequently,
\[
    \Expt\!\big[
        \big(
            \hat A^{\mathrm{boost}}_{i,\tau}
            -
            A_i^{\vpi}(s_\tau,a_\tau)
        \big)^2
        \,\big|\, s_\tau,a_\tau
    \big]
    \le
    (1+\gamma)^2
    \epsilon_Q^2 \cdot 
    \bigg(
        \sum_{k=0}^{L_\tau-1}
        \rho^k
    \bigg)^2 .
\]
Since $\Gamma^\lambda_{i,\tau}>0$, define
\[
    \xi^\lambda_{i,\tau}
    \triangleq
    \frac{
        \sqrt{\Gamma^\lambda_{i,\tau}}
    }{
        (1+\gamma)
        \sum_{k=0}^{L_\tau-1}
        (\lambda\gamma)^k
    }.
\]
This quantity is strictly positive. If
\[
    \lVert Q_i(\cdot;\phi_i)-Q_i^{\vpi}\rVert_\infty
    =
    \epsilon_Q
    <
    \xi^\lambda_{i,\tau},
\]
then
\[
\begin{aligned}
    \Expt\!\big[
        \big(
            \hat A^{\mathrm{boost}}_{i,\tau}
            -
            A_i^{\vpi}(s_\tau,a_\tau)
        \big)^2
        \,\big|\, s_\tau,a_\tau
    \big]
    &<
    \Gamma^\lambda_{i,\tau} =
    \Expt\!\big[
        \big(
            \hat A^{\mathrm{GAE}}_{i,\tau}
            -
            A_i^{\vpi}(s_\tau,a_\tau)
        \big)^2
        \,\big|\, s_\tau,a_\tau
    \big].
\end{aligned}
\]
Thus, for any $\lambda\in(0,1]$, a sufficiently accurate $Q$-critic
yields strictly smaller conditional mean-squared advantage-estimation
error than exact-$V$ GAE whenever the reachable future actions are not
all $Q_i^{\vpi}$-value-equivalent. The theorem follows.
\end{proof}

\paragraph{Remarks.}
We remark on the role of the trace parameter. For $\lambda=1$, on-policy $Q$-boosting gives an unbiased score-function term for any fixed $Q$-critic. With an exact $Q$-critic, $Q$-boosting equals the exact advantage for every $\lambda$, since every Expected SARSA residual vanishes pathwise. With an approximate critic and $\lambda<1$, bias may remain, and the theorem gives the lower conditional mean-squared error for every $\lambda\in(0,1]$.

The radius $\xi^\lambda_{i,\tau}$ is sufficient but not necessary, and it scales with the future-action noise that $Q$-boosting removes. In the matching-pennies example of Figure~\ref{fig:gae-vs-qboosting}, we have $\gamma=\lambda=1$, $\Gamma^\lambda_{i,\tau}=1$, and $L_\tau=2$, which gives $\xi^\lambda_{i,\tau}=1/4$ in reward units.

\clearpage

\section{Games Description} \label{sec:game-description}
\paragraph{Abrupt Phantom Tic-Tac-Toe (APTTT).}
Phantom Tic-Tac-Toe is an imperfect-information version of standard Tic-Tac-Toe on a $3\times 3$ grid: players alternately attempt to place their mark, aiming to achieve three in a row horizontally, vertically, or diagonally. If the grid fills without a winning line, the game is a draw. In Phantom Tic-Tac-Toe, imperfect information arises because the opponent's marks are hidden, players do not directly observe the opponent's moves, and a player only learns whether their own attempted placement succeeds or fails. The \emph{abrupt} variant differs from the classical ruleset in how failed placements are handled: when a player attempts to place a mark on a square occupied by the opponent, the attempt fails and the player \emph{loses their turn} (in the classical ruleset, the player must choose a different square) \citep{lanctot2019openspiel, rudolph2025reevaluating}.

\paragraph{Abrupt Dark Hex.}
Abrupt Dark Hex is an imperfect-information variant of Hex. As in Hex, the first player aims to connect the top and bottom edges, while the second player aims to connect the left and right edges; draws are impossible (Hex theorem). The game becomes imperfect-information because actions are \emph{hidden}: each player observes only their own stones, and the opponent's placements are not directly visible. If a player attempts to place a stone on a cell already occupied by the opponent, the move fails. In the \emph{abrupt} ruleset, the player \emph{loses their turn} (instead of retrying a different cell, as in the classical ruleset) \citep{lanctot2019openspiel, rudolph2025reevaluating}. We use \textbf{Abrupt Dark Hex 3 (ADH3)}, the variant played on a $3\times 3$ board. This is the largest size for which we can evaluate exact policy exploitability within a reasonable computational budget.

\paragraph{Liar's Dice.}
In Liar's Dice, each of two players privately rolls dice. Players then alternate making bids of the form ``$(q,f)$,'' claiming that across both players there are at least $q$ dice showing face $f$. Under the OpenSpiel ruleset, a legal bid must either increase the quantity $q$, or increase the face value $f$ without decreasing $q$. Instead of bidding, a player may call \emph{liar}, which ends the round by revealing all dice: if the previous bid does not exceed the actual count of face $f$, the bidder wins; otherwise, the caller wins. Liar's Dice (including closely related parameterizations) is a standard benchmark for imperfect-information games in prior equilibrium-finding and RL+search work \citep{lanctot2019openspiel, brown2020combining}. We use \textbf{Liar's Dice with 2 Dice of 5 Faces (LiDi 2D5F)}, a variant in which both players roll two dice with five faces each. This is the largest size for which we can compute exact exploitability within the available computational budget.

\paragraph{Dou Dizhu.}
Dou Dizhu (\emph{Fighting the Landlord}) is a three-player imperfect-information shedding game played with a 54-card poker deck, between one landlord and two peasants. Players take turns playing legal card combinations (e.g., singles, pairs, triples, sequences) that must beat the previous play or pass. The goal is to be the first to empty one's hand. The landlord wins if they go out first, while the peasants win if \emph{either} peasant empties their hand first. Rewards are computed using a base score of $2$ for the landlord and $1$ for each peasant if they win, with additional bonuses based on the number of special combinations played: \emph{bombs} (four of a kind) and \emph{rocket} (the two jokers). Dou Dizhu is widely used as a benchmark for self-play RL and imperfect-information decision-making, including end-to-end systems such as DouZero \citep{zha2021douzero} and toolkits/environments such as RLCard \citep{zha2019rlcard}.

\paragraph{Heads-Up No-Limit Texas Hold'em.}
Heads-Up No-Limit Texas hold'em is a two-player imperfect-information betting game. Each hand begins with forced blinds and a shuffle deal: each player receives two private \emph{hole cards}, and up to five \emph{community cards} are revealed over four betting rounds (preflop, flop, turn, river). In each round, players act sequentially and may fold, call/check, or bet/raise any amount up to their remaining stack (the \emph{no-limit} betting structure). The hand ends either when a player folds, awarding the current pot to the opponent, or at showdown after the final betting round, where private cards are revealed and the stronger five-card poker hand wins the pot. Imperfect information arises from the hidden hole cards and stochastic dealing, making HUNL a standard benchmark for equilibrium-finding and decision-making under uncertainty in poker agent \citep{moravvcik2017deepstack, brown2017libratus, lanctot2019openspiel}. We use \textbf{HUNL with a 200 big-blind initial stack (HUNL200)}, where each player starts each hand with a stack of $200$ big blinds of $20{,}000$ chips, with small blind and big blind equal to $50$ chips and $100$ chips respectively. Specifically, to prevent excessively long trajectories in the RL environment, the fifth raise within a betting street is forced to be an all-in.

\clearpage

\section{Training Recipe} \label{sec:training}

\paragraph{RL Environment Overview.}
The RL environment provides, at each timestep:
\begin{itemize}[leftmargin=*]
    \item Per-player tokens describing the differences from the previous step. Notably, the full sequence of tokens over timesteps should contain all information observed by the player. For APTTT or ADH3, this encodes the last move made by the player, as well as whether the placement was successful. For Dou Dizhu games, this encodes the most recently drawn card during the card drawing phase, or the selected card combinations of the active player.
    
    \item Per-player binary feature channels representing the current view of the board. These channels do not encode the full board state. Although the sequence of channels across timesteps contains the complete information, we only use the current timestep's channels for efficiency. For APTTT or ADH3, this indicates whether each square has been played successfully, unsuccessfully, or remains untouched. For Dou Dizhu games, this encodes the player's current hand as well as the cards discarded by each player.
    
    \item Candidate action sets for the active player, each encoded as a variable-length token sequence describing a possible move. We use a similar format as the observation tokens. For complex decision spaces, such as in Dou Dizhu (which may have more than a thousand valid moves), actions are decomposed over multiple timesteps to reduce the candidate set size. For APTTT or ADH3, this consists of all unplayed squares encoded by their positions. For Dou Dizhu, each action is represented as a \emph{prefix} of a valid card combination, breaking the decision into multiple steps.
\end{itemize}

This design is based on the observation that changes between timesteps are relatively small compared to the full game state. Therefore, we avoid encoding the entire game state at each step and instead encode only the changes, supplemented by a current board description to help the model reconstruct the full state. Ideally, this design can be seen as a backtrace of the game state. We use token-based descriptions since the changes are often sparse, making binary channels inefficient.

\paragraph{Training Hyperparameters.}

We use the Muon optimizer~\citep{jordan2024muon} with Nesterov momentum $\beta = 0.95$ and a base learning rate of $\eta_{\mathrm{base}} = 4 \times 10^{-4}$. We apply fixed weight decay $0.01$ to all projection matrices from the Transformer decoder and the MLPs, where RMSNorm is attached to the output of each layer, making the architecture homogeneous.

The algorithm performs $K_{\mathrm{actor}} = 4$ actor epochs and $K_{\mathrm{critic}} = 4$ critic epochs per iteration, each consisting of $M = 4$ minibatches. It employs a base PPO clipping coefficient of $\varepsilon_{\mathrm{base}} = 0.02$ and a base regularization coefficient of $\alpha_{\mathrm{base}} = 0.1$. The rollout batch size is denoted as $|\cD_{\mathrm{rollout}}| = B$, and the replay buffer has a capacity of $|\cD_{\mathrm{replay}}| = 64 \cdot B$, sufficient to store all trajectories from the past 64 iterations. We present an ablation study of VRPO-specific hyperparameters in Appendix~\ref{sec:ablate-vrpo}.
The discount factor of all games is fixed at $\gamma = 1$, and the $Q$-boosting trace parameter is set to $\lambda = 0.95$. 

In addition to the total number of iterations $T_{\mathrm{total}}$, the learning rate and regularization schedules are controlled by two decay coefficients: the stable learning rate iteration count $T_{\eta}$ and the stable regularization coefficient iteration count $T_{\alpha}$. The schedule hyperparameters are game-specific and listed in Table~\ref{tab:schedule}.

\begin{table}[!h]
\centering
\small
\setlength{\tabcolsep}{6pt}
\renewcommand{\arraystretch}{1.15}
\begin{tabular}{@{}lcccc@{}}
\toprule
\textbf{Game} & Iteration $T_{\mathrm{total}}$ &  Iteration $T_{\eta}$ & Iteration $T_{\alpha}$ & Batch Size $B$ \\
\midrule
APTTT & $2000$ & $500$ & $500$ & $2048$ \\
ADH3  & $2000$ & $500$ & $500$ & $2048$ \\
LiDi 2D5F & $2000$ & $500$ & $500$ & $2048$ \\
Dou Dizhu & $40000$ & $5000$ & $500$ & $2048$ \\
HUNL200 & $40000$ & $5000$ & $500$ & $8192$ \\
\bottomrule
\end{tabular}
\vspace{6pt}
\caption{Schedule hyperparameters by game.}
\vspace{-12pt}
\label{tab:schedule}
\end{table}

At iteration $T$, the actor learning rate, critic learning rate, and PPO clip coefficient are decayed according to:
\[
    \eta_{\mathrm{actor}} = \eta_{\mathrm{base}} \cdot \min\bigg\{1, \frac{T_{\eta}}{T} \bigg\}, \qquad
    \eta_{\mathrm{critic}} = \eta_{\mathrm{base}} \cdot \min\bigg\{1, \frac{T_{\eta}}{T} \bigg\}^{0.5}, \qquad 
    \varepsilon = \varepsilon_{\mathrm{base}} \cdot \min\bigg\{1, \frac{T_{\eta}}{T} \bigg\}.
\]
The regularization coefficient decays according to:
\[
    \alpha = \alpha_{\mathrm{base}} \cdot \min\bigg\{1, \frac{T_{\alpha}}{T} \bigg\}^{0.5}.
\]

For Dou Dizhu and HUNL200, both of which require many training iterations, we follow \citet{sokota2025superhuman} by using an exponential moving average of the policy with $\beta=0.999$ as the evaluation policy, effectively merging the policies from the final several thousand iterations. This model-merged policy is not used for training rollouts; it is used only for evaluation. This averaging smooths out the noise introduced by rollout batches at individual gradient steps.

\paragraph{Model Architecture.}

We use the same model architecture across all games and algorithms. At each timestep, observation tokens are embedded into a $128$-dimensional space, summed over all available tokens, and combined with a learned player-identity embedding. The resulting temporal sequence is processed by a small LLaMA-style causal (decoder-only) Transformer~\citep{touvron2023llama} with $4$ layers, $4$ attention heads, and a feedforward expansion ratio of $2$. The decoder outputs are then concatenated with the channel features and passed through a $3$-layer MLP with hidden dimension $256$, equipped with RMSNorm, to produce a per-player state feature.

Each candidate action is encoded independently by summing its token embeddings (each of dimension $256$) and applying an action MLP, similar to the state MLP but with only $2$ layer. Action scores are computed by summing the elementwise product between the $256$-dimensional state feature and each action feature, followed by a learned linear projection. The actor architecture is shown in Figure~\ref{fig:arch}.

\begin{figure}[!h]
    \centering
    \scalebox{0.8}{
\begin{tikzpicture}[
    every node/.style={font=\small},
    module/.style={
        rectangle,
        draw,
        thick,
        minimum width=72pt,
        rounded corners,
    },
    operator/.style={
        circle,
        draw,
        thick,
        minimum width=10pt,
    },
    stateModule/.style={fill=fancyBlue!40},
    actionModule/.style={fill=fancyRed!40},
    finalModule/.style={fill=fancyColor!40},
    h1module/.style={module, minimum height=18pt},
    h2module/.style={module, minimum height=36pt},
    line/.style={-, thick},
    trans/.style={->, shorten >=1pt, >={Stealth[round]}, thick},
]
    
    \node [anchor=center] at (0, -7) (actor-logits) {action logits};
    
    \node [operator, finalModule, opacity=0.36] at ($(0, -6) + (14pt, 10pt)$) (actor-head-2) {};
    \draw [trans, opacity=0.36] (actor-head-2) to ([xshift=14pt] actor-logits.north);
    \node [operator, finalModule, opacity=0.6] at ($(0, -6) + (7pt, 5pt)$) (actor-head-1) {};
    \draw [trans, opacity=0.6] (actor-head-1) to ([xshift=7pt] actor-logits.north);
    \node [operator, finalModule] at (0, -6) (actor-head) {};
    \draw [trans] (actor-head) to (actor-logits);
    \node at ($(actor-head)$) {$\cdot$};
    \node [anchor=north east, opacity=0.5] at ($(actor-head) + (-2pt, 0)$) {inner product};

    \node [h2module, align=center, stateModule] at (-2.4, -4.5) (state-mlp) {State MLP};
    \draw [line, opacity=0.36] (state-mlp) |- (actor-head-2);
    \draw [line, opacity=0.6] (state-mlp) |- (actor-head-1);
    \draw [line] (state-mlp) |- (actor-head);
    \node [anchor=north east, opacity=0.5] at ($(state-mlp) + (0, -20pt)$) {256 feats};

    \node [operator, stateModule] at (-2.4, -3) (state-concat) {};
    \node [anchor=north east, opacity=0.5] at ($(state-concat) + (-4pt, 0)$) {concat};
    \draw [trans] (state-concat) to (state-mlp);

    \node [h2module, align=center, stateModule] at (-2.4, -1.5) (state-decoder) {Causal \\ Transformer};
    \draw [trans] ($(state-decoder)+(-48pt,0)$) to (state-decoder);
    \draw [trans] (state-decoder) to ($(state-decoder)+(48pt, 0)$);
    \draw [trans] (state-decoder) to (state-concat);
    \node [anchor=north east, opacity=0.5] at ($(state-decoder) + (0, -20pt)$) {128 feats};

    \node [anchor=center] at (-2.4, 0) (state-tokens) {SB raises to 4};
    \node [anchor=north east, opacity=0.5] at ($(state-tokens) + (0, -6pt)$) {tokens};
    \draw [trans] (state-tokens) to (state-decoder);
    
    \node [anchor=center] at (0, 0) (state-channel) {total pot 6, ...};
    \draw [trans] (state-channel) |- (state-concat);
    \node [anchor=north west, align=left, opacity=0.5] at ($(state-channel) + (0, -6pt)$) {binary \\ feats};
    
    \node [anchor=center, opacity=0.5] at ($(-1.2, 0) + (0, 12pt)$)  {state representations};

    \node [h2module, align=center, actionModule, opacity=0.36] at ($(2.4, -4.5) + (14pt, 10pt)$) (action-mlp-2) {};
    \draw [line, opacity=0.36] (action-mlp-2) |- (actor-head-2);
    \node [h2module, align=center, actionModule, opacity=0.6] at ($(2.4, -4.5) + (7pt, 5pt)$) (action-mlp-1) {};
    \draw [line, opacity=0.6] (action-mlp-1) |- (actor-head-1);
    \node [h2module, align=center, actionModule] at (2.4, -4.5) (action-mlp) {Action MLP};
    \draw [line] (action-mlp) |- (actor-head);

    \node [anchor=center] at (2.4, 0) (action-tokens) {BB calls};
    
    \draw [trans, opacity=0.36] ([xshift=14pt] action-tokens.south) to (action-mlp-2);
    \draw [trans, opacity=0.6] ([xshift=7pt] action-tokens.south) to (action-mlp-1);
    \draw [trans] (action-tokens) to (action-mlp);

    \node [anchor=center, opacity=0.5] at ($(2.4, 0) + (0, 12pt)$) {available actions};
    
\end{tikzpicture}
}
    \caption{
        Actor architecture.
        The observation history is encoded into features by a decoder-only Transformer, combined with binary channel features, and scored against independently encoded action features via inner product to produce action logits.
        The critic is built on the actor architecture, with an additional residual MLP after the state MLP to fuse state features from all players. The architecture is shared across all games; see Appendix~\ref{sec:training} for details.
    }
    \label{fig:arch}
\end{figure}
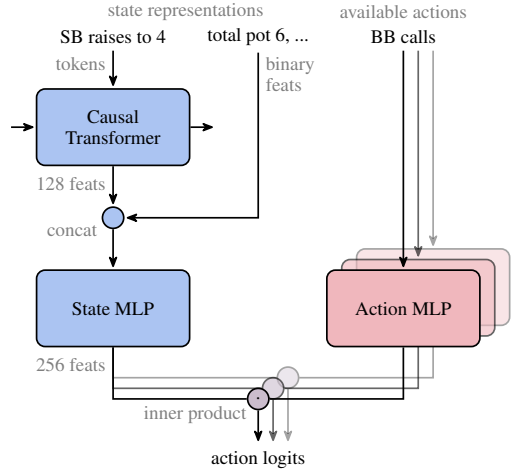

The critic mirrors the per-player history/state encoder for all players (with separate parameters), fuses the concatenated per-player state features using a $2$ layer residual MLP mixer of dimension $256$, and outputs per-player values for each candidate action using the same state-action interaction and a linear output head. In models where only a value head is required, the elementwise product is replaced by the state feature alone.

\clearpage
\section{Evaluation Details} \label{sec:evaluation}
\subsection{Abrupt Phantom Tic-Tac-Toe and Abrupt Dark Hex 3}

We use the \texttt{exp-a-spiel} library~\citep{rudolph2025reevaluating}, which employs a traverser to compute the exact exploitability from the tabular strategy. The tabular strategy is generated by enumerating all information sets and evaluating the actor network to obtain the policy distribution from model logits at temperature $\tau = 1$.

\subsection{Liar's Dice}

We implement a traverser routine to compute the exact exploitability from the tabular strategy. This routine is cross-verified against the \texttt{OpenSpiel} library~\citep{lanctot2019openspiel} on smaller Liar's Dice games. The tabular strategy is generated by enumerating all information sets and evaluating the actor network to obtain the policy distribution from model logits at temperature $\tau = 1$.

\subsection{Dou Dizhu}

We integrate our model with the evaluation routine of PerfectDou~\citep{yang2022perfectdou} by converting their encoding to a semantic string representation that describes the information set. This is then converted into the game encoding presented in Appendix~\ref{sec:game-description}. The semantic string encoding creates an information barrier that prevents potential information leakage. We compute our policy by evaluating the actor network to obtain the policy distribution from model logits at temperature $\tau = 0$.

\subsection{No Limit Texas Hold'em}

We evaluate the performance by querying the Slumbot API.\footnote{https://slumbot.com/} We compute our policy by evaluating the actor network to obtain the policy distribution from model logits at temperature $\tau = 0.1$. 

We remark that the external regularization term $\cL^{\mathrm{reg}}$ introduces additional exploration. For example, the training process must handle all-in opens, even though these actions are far from equilibrium play. Reducing the inference temperature is the simplest way to eliminate such actions from the agent's behavior. Clean-up using human knowledge~\citep{perolat2022mastering} or test-time search~\citep{sokota2025superhuman} can refine this step further.

\clearpage
\section{Game Encoding} \label{sec:game-definition}

In this section, we present the game encoding used in the RL environment. An overview of the hyperparameters is shown in Table~\ref{tab:game-parameters}, where \textbf{Tokens} refers to the number of tokens in the dictionary, including the special padding token \texttt{Token~0}; \textbf{Channels} refers to the number of binary channels used to encode the current timestep; \textbf{Actions} refers to the maximum number of available actions in a single timestep; \textbf{Horizon} represents the maximum possible environment steps according to the encoding; \textbf{O tokens} indicates the maximum number of observation tokens shown in a single timestep; and \textbf{A tokens} denotes the maximum number of tokens used in an action.

\begin{table}[!ht]
\centering
\small
\setlength{\tabcolsep}{6pt}
\renewcommand{\arraystretch}{1.15}
\begin{tabular}{@{}lcccccc@{}}
\toprule
\textbf{Game} & Tokens & Channels & Horizon & Actions & O Tokens & A Tokens \\
\midrule
APTTT & $10$ & $29$ & $18$ & $9$ & $3$ & $2$ \\
ADH3  & $10$ & $29$ & $18$ & $9$ & $3$ & $2$ \\
LiDi 2D5F & $16$ & $23$ & $13$ & $11$ & $4$ & $4$ \\
Dou Dizhu & $91$ & $288$  & $177$ & $35$ & $5$ & $5$ \\
HUNL200 & $77$ & $297$ & $100$ & $16$ & $8$ & $8$ \\
\bottomrule
\end{tabular}
\vspace{6pt}
\caption{Game encoding parameters.}
\vspace{-6pt}
\label{tab:game-parameters}
\end{table}

Importantly, for games where the number of available actions is large in the original game logic, such as Dou Dizhu or HUNL, we divide a single decision into multiple environment steps to accommodate the maximum number of available actions within the game horizon.

\subsection{Abrupt Phantom Tic-Tac-Toe and Abrupt Dark Hex 3}

\paragraph{Token Dictionary.}
\begin{itemize}[leftmargin=*]
    \item \texttt{Token~1}: Initial token
    \item \texttt{Token~2--4}: Row token $0 \dots 2$
    \item \texttt{Token~5--7}: Column token $0 \dots 2$
    \item \texttt{Token~8--9}: Player token $0 \dots 1$
\end{itemize}

\paragraph{Binary Channels.}
\begin{itemize}[leftmargin=*]
    \item \texttt{Channel 0--8}: Unobserved cells: $0 \dots 8$
    \item \texttt{Channel 9--17}: Cells $0 \dots 8$ occupied by player~0
    \item \texttt{Channel 18--26}: Cells $0 \dots 8$ occupied by player~1
    \item \texttt{Channel 27--28}: Turn indicator for player $0 \dots 1$
\end{itemize}

\paragraph{Observation Tokens.}
The game starts with an initial token (\texttt{Token~1}). Then, after each move made by the player, they observe three tokens: a row token \texttt{r}, a column token \texttt{c}, and a player token \texttt{i}. These tokens indicate that playing at the corresponding cell was successful if the player token matches the current player, or that the cell was already occupied by the other player if the player token belongs to them. After each opponent's move, the player does not observe anything.

\paragraph{Action Tokens.}
Each available action corresponds to a placement at an untried cell and is encoded by the corresponding row token \texttt{r} and column token \texttt{c}.

\paragraph{Rewards.}
The winning agent is awarded $+1$, the losing agent is awarded $-1$, and both receive $0$ in the case of a draw.

\subsection{Liar's Dice with 2 Dice of 5 Faces}

\paragraph{Token Dictionary.}
\begin{itemize}[leftmargin=*]
    \item \texttt{Token~1}: Initial token
    \item \texttt{Token~2--3}: Player token $0 \dots 1$
    \item \texttt{Token~4}: Bid token
    \item \texttt{Token~5}: Call token
    \item \texttt{Token~6--10}: Face tokens for $1 \dots 5$
    \item \texttt{Token~11--15}: Quantity tokens for $0 \dots 4$
\end{itemize}

\paragraph{Binary Channels.}
\begin{itemize}[leftmargin=*]
    \item \texttt{Channel 0--10}: Last bid: $1 \dots 4$ of face $1 \dots 5$, or no bid yet
    \item \texttt{Channel 11--15}: Private die~0 shows face $1 \dots 5$
    \item \texttt{Channel 16--20}: Private die~1 shows face $1 \dots 5$
    \item \texttt{Channel 21--22}: Turn indicator for player $0 \dots 1$
\end{itemize}

\paragraph{Observation Tokens.}
The game starts with an initial token (\texttt{Token~1}). For each dice roll, the player observes two tokens: a face token \texttt{f} and a player token \texttt{i}, indicating the result of each die. After any player's bid, the player observes four tokens: a player token \texttt{i}, the bid token (\texttt{Token~4}), a quantity token \texttt{c}, and a face token \texttt{f}. These tokens indicate that the corresponding player is making a bid for a certain quantity of the specified face.

\paragraph{Action Tokens.}
Each bidding action corresponds to four tokens: a player token \texttt{i}, the bid token (\texttt{Token~4}), a quantity token \texttt{c}, and a face token \texttt{f}, which exactly match the observation tokens. The call action consists of two tokens: a player token \texttt{i} and the call token (\texttt{Token~5}).

\paragraph{Rewards.}
The winning agent is awarded $+1$, the losing agent is awarded $-1$.

\subsection{Dou Dizhu}

\paragraph{Token Dictionary.}
\begin{itemize}[leftmargin=*]
    \item \texttt{Token~1}: Initial token
    \item \texttt{Token~2--4}: Player tokens -- Landlord, Peasant 1, Peasant 2
    \item \texttt{Token~5--8}: Discard Class -- Solo, Pair, Trio, or Quad
    \item \texttt{Token~9--11}: Discard Class -- Straight, Pair Chain (Pair Sisters), or Trio Chain
    \item \texttt{Token~12}: Discard Class -- Rocket
    \item \texttt{Token~13--27}: Rank tokens -- 3, 4, 5, 6, 7, 8, 9, T, J, Q, K, A, 2, B, R
    \item \texttt{Token~28--38}: Sequence starting point -- 3, 4, 5, 6, 7, 8, 9, T, J, Q, K
    \item \texttt{Token~40--50}: Sequence end point -- 4, 5, 6, 7, 8, 9, T, J, Q, K, A
    \item \texttt{Token~51}: No kickers
    \item \texttt{Token~52--56}: Solo kickers with $1 \dots 5$ cards
    \item \texttt{Token~57--60}: Pair kickers with $1 \dots 4$ cards
    \item \texttt{Token~89}: Pass token
    \item \texttt{Token~90}: Hold token\footnote{Remark: Some token indices are left empty due to internal backward compatibility.}
\end{itemize}

\paragraph{Binary Channels.}
\begin{itemize}[leftmargin=*]
    \item \texttt{Channel~0--53}: Self-remaining cards for each card
    \item \texttt{Channel~54--107}: Landlord public cards for each card
    \item \texttt{Channel~108--161}: Landlord discarded cards for each card
    \item \texttt{Channel~162--215}: Peasant~1 discarded cards for each card
    \item \texttt{Channel~216--269}: Peasant~2 discarded cards for each card
    \item \texttt{Channel~270--284}: Prefix-filling bomb count
    \item \texttt{Channel~285--287}: Turn indicators for Landlord, Peasant 1, and Peasant 2
\end{itemize}

\paragraph{Observation Tokens.}
The game starts with an initial token (\texttt{Token~1}). In the next $17$ timesteps, each of the $17$ private cards is revealed to each player privately using three tokens: a player token \texttt{i}, the hold token (\texttt{Token~90}), and a rank token $r$. In the following $3$ timesteps, the three Landlord public cards from the remaining of the deck are announced to all players using three tokens: the Landlord token (\texttt{Token~2}), the hold token (\texttt{Token~90}), and a rank token~$r$.

During gameplay, the card combinations are encoded across multiple timesteps, depending on the complexity of the card combination.
\begin{itemize}[leftmargin=*]
    \item \textbf{Solo/Pair} combinations are encoded in one timestep with three tokens: a player token~$i$, the Solo token (\texttt{Token~5}) or the Pair token (\texttt{Token~6}), and the rank token~$r$.
    \item \textbf{Trio/Quad} combinations are encoded in multiple timesteps. The first timestep contains three tokens: a player token~$i$, the Trio token (\texttt{Token~7}) or the Quad token (\texttt{Token~8}), and the major rank token~$r$. In the second timestep, the kicker details are specified, and the token count increases to four by including the kicker information. This can be no kicker (\texttt{Token~51}), one solo kicker (\texttt{Token~52}) or one pair kicker (\texttt{Token~57}) for a trio, or two solo kickers (\texttt{Token~53}) or two pair kickers (\texttt{Token~58}) for a quad. The following timesteps encode each kicker as if they were discarded individually as a solo or pair, in increasing order.
    \item \textbf{Straight/Pair Chain} combinations are encoded in two timesteps. The first timestep contains three tokens: a player token~$i$, the Straight token (\texttt{Token~9}) or the Pair Chain token (\texttt{Token~10}), and the sequence starting point token~$r_1$. In the next timestep, the end point of the sequence is determined, and four tokens are presented, including the sequence end point token~$r_2$.
    \item \textbf{Trio Chain} combinations are encoded in multiple timesteps. The first timestep contains three tokens: a player token~$i$, the Trio Chain token (\texttt{Token~11}), and the sequence starting point token~$r_1$. In the second timestep, both the kicker information and the sequence end point are specified together. For a Trio Chain with kickers (airplane with wings), the token sequence contains five tokens, including the sequence end point token~$r_2$ and the kicker information. The following timesteps encode each kicker as if it were discarded individually as a solo or pair, in increasing order.
    \item \textbf{Rocket} is encoded in one timestep with two tokens: a player token~$i$ with Rocket token (\texttt{Token~12}).
    \item \textbf{PASS} is encoded in one timestep with two tokens: a player token~$i$ and the pass token (\texttt{Token~89}).
\end{itemize}
Under this multi-timestep decision encoding, the maximum number of actions required in one timestep is significantly reduced to $35$ actions, compared to $27{,}472$ in DouZero \citep{zha2021douzero}. Notably, the channels are updated as if the kickers were treated as continuous actions by the player: a Trio Chain with three pair kickers is treated as first discarding a Trio Chain with kickers, followed by three additional timesteps for discarding each kicker, which are limited to pairs.

\paragraph{Action Tokens.}
The action tokens are identical to the observation tokens as the actions are public.

\paragraph{Rewards.}
Let $B$ be the total number of quads and rockets played during gameplay. The Landlord is awarded $+2 \times 2^B$ if they win, and $-2 \times 2^B$ if they lose. Each Peasant is awarded $+1 \times 2^B$ if they win, and $-1 \times 2^B$ if they lose.

\subsection{Heads-Up No-Limit Texas Hold'em}

\paragraph{Token Dictionary.}
\begin{itemize}[leftmargin=*]
    \item \texttt{Token~1}: Initial token
    \item \texttt{Token~2--14}: Rank tokens -- 2, 3, 4, 5, 6, 7, 8, 9, T, J, Q, K, A
    \item \texttt{Token~15--18}: Suit tokens -- c, d, h, s
    \item \texttt{Token~19--21}: Action tokens -- Fold, Call, Raise
    \item \texttt{Token~22--25}: Street tokens -- Preflop, Flop, Turn, River
    \item \texttt{Token~26--27}: Player tokens -- Small Blind, Big Blind
    \item \texttt{Token~28--75}: Amount tokens -- Raise to $i \times 16^j \times 10$ chips ($0.1$ big blinds)
    \item \texttt{Token~76}: Raise all-in
\end{itemize}

The raise tokens encode 3 base-16 digits, specifically following amounts: 0, 10, 20, 30, \dots, 150, 160, 320, 480, \dots, 2560, 5120, 7680, \dots, 40960 chips. 

\paragraph{Binary Channels.}
\begin{itemize}[leftmargin=*]
    \item \texttt{Channel~0--33}: $2 \times 17$ channels for the two hole cards: suits c, d, h, s then ranks 2..A
    \item \texttt{Channel~34--118}: $5 \times 17$ channels for the five board cards, same as private hands
    \item \texttt{Channel~119-170}: self hole cards, 52-card multi-hot
    \item \texttt{Channel~171--222}: visible board cards, 52-card multi-hot
    \item \texttt{Channel~223--246}: Total pot, 6 base-4 digits
    \item \texttt{Channel~247--270}: Current player's bet on this street, 6 base-4 digits
    \item \texttt{Channel~271--294}: Opponent's player's bet on this street, 6 base-4 digits
    \item \texttt{Channel~295--296}: Turn indicator for Small Blind and Big Blind
\end{itemize}

\paragraph{Observation Tokens.}
In the encoding, each card dealt is encoded over two timesteps: the rank token is presented in the first timestep, while the suit token is presented in the second timestep. This reduces the number of branches from 52 cards to 13 ranks, as we treat the board dealer as a special player for deterministic transitions. 

The game starts with an initial token (\texttt{Token~1}). The next $2 \times 2$ timesteps reveal the two hole cards to each player \emph{simultaneously}. After each betting round, the corresponding board cards are revealed, with each card taking two timesteps. 

For a board card, the first timestep contains one token, the rank token $r$, while the second timestep contains two tokens: the rank token $r$ and the suit token $s$. For a private hole card, the player token $i$ is also attached to the encoding, so the number of tokens becomes two and three respectively.

In the betting rounds, actions are encoded over multiple timesteps depending on the complexity of the action. For a raise action, the sequence is encoded using three timesteps, each timestep decides one digit of the base-16 numbers. In an all-in raise, the raise sequence is replaced by the all-in token. 

The other actions, Call/Check and Fold, are encoded using three tokens in one timestep, containing a street token, a player token $i$ and either the Call token (\texttt{Token~20}) or the Fold token (\texttt{Token~19}).

The environment faithfully represents all trajectories, since actions are never aggregated. The only exception is that, within each betting street, the fifth raise is constrained to be an all-in raise to avoid excessively long training episodes.

\paragraph{Action Tokens.}
The action tokens are identical to the observation tokens as the actions are public.

\paragraph{Rewards.}
The players are awarded a chip change measured in big blinds. Since there are $20{,}000$ chips in the stack, the maximum reward is $200$.

\clearpage
\section{Training Dynamics}
\label{sec:dynamics}

In this section, we present the training dynamics of VRPO on Dou Dizhu and HUNL200, using the default hyperparameters from the training recipe in Appendix~\ref{sec:training}. The solid curve represents an exponential moving average of the metric over training iterations, while the faint shaded trace corresponds to the raw, unsmoothed values.

\begin{figure}[!h]
    \centering
    \includegraphics[scale=0.8]{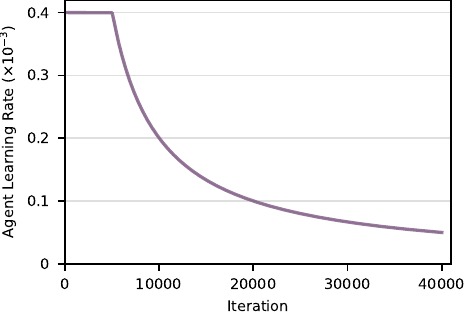}
    \hspace{8pt}
    \includegraphics[scale=0.8]{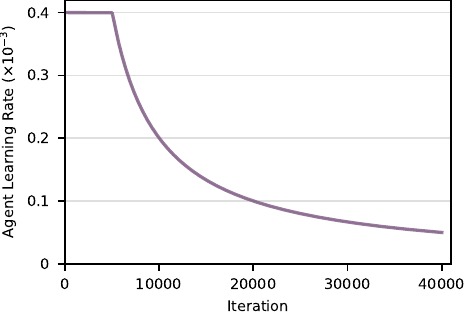}
    \caption{Actor learning rate schedule $\eta_{\mathrm{actor}}$ (Left: Dou Dizhu; Right: HUNL200).}
\end{figure}

\begin{figure}[!h]
    \centering
    \includegraphics[scale=0.8]{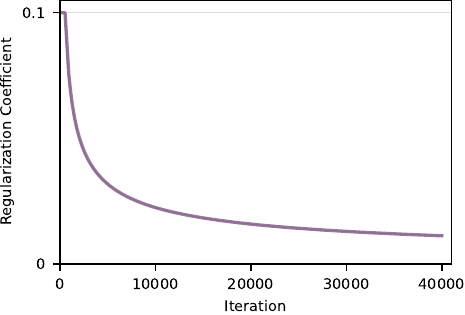}
    \hspace{8pt}
    \includegraphics[scale=0.8]{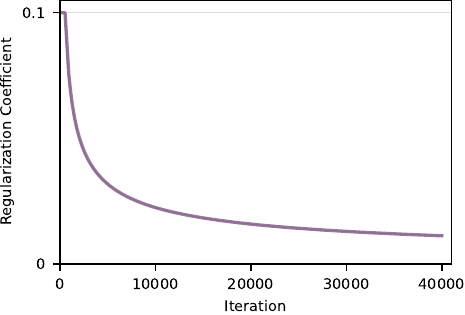}
    \caption{Regularization coefficient schedule $\alpha$ (Left: Dou Dizhu; Right: HUNL200).}
\end{figure}

\begin{figure}[!h]
    \centering
    \includegraphics[scale=0.8]{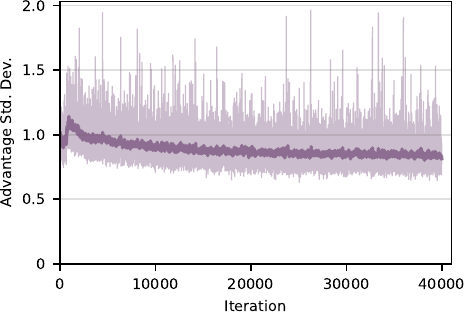}
    \hspace{8pt}
    \includegraphics[scale=0.8]{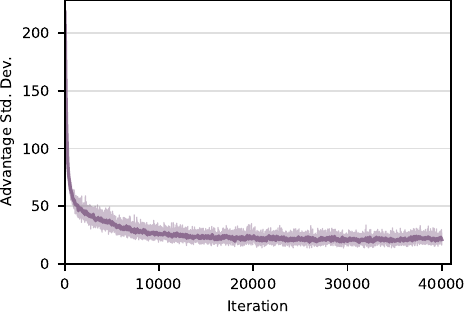}
    \caption{Standard deviation of the estimated advantage $\hat{A}$ (Left: Dou Dizhu; Right: HUNL200).}
\end{figure}

\clearpage

\begin{figure}[!h]
    \centering
    \includegraphics[scale=0.8]{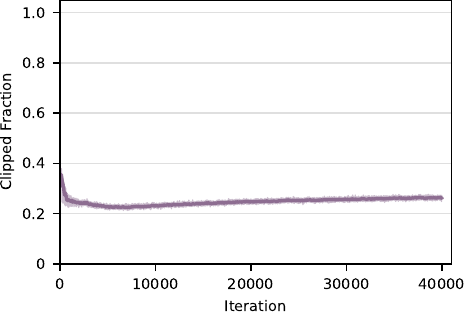}
    \hspace{8pt}
    \includegraphics[scale=0.8]{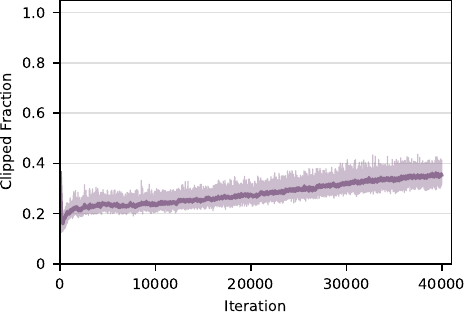}
    \caption{Clipped fraction under the PPO threshold (Left: Dou Dizhu; Right: HUNL200).}
\end{figure}

\begin{figure}[!h]
    \centering
    \includegraphics[scale=0.8]{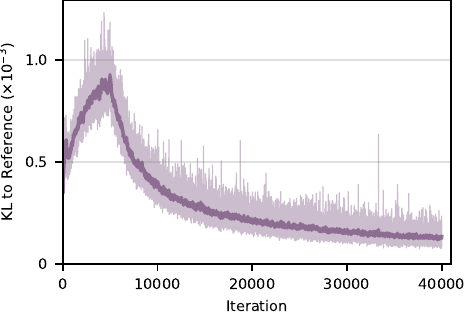}
    \hspace{8pt}
    \includegraphics[scale=0.8]{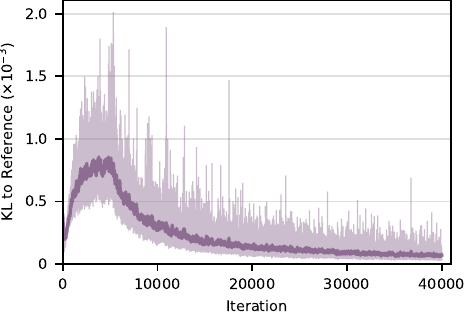}
    \caption{KL divergence to the reference policy $\mathrm{KL}(\vpi \mmid \vpi^{\mathrm{ref}})$ (Left: Dou Dizhu; Right: HUNL200).}
\end{figure}

\begin{figure}[!h]
    \centering
    \includegraphics[scale=0.8]{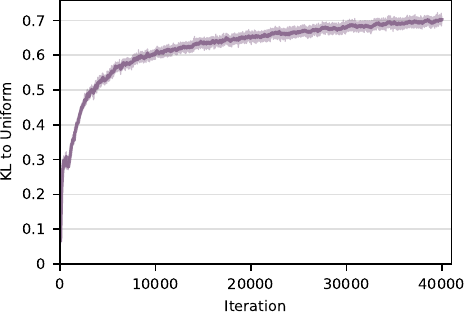}
    \hspace{8pt}
    \includegraphics[scale=0.8]{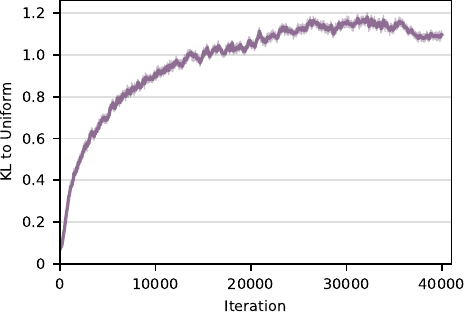}
    \caption{KL divergence to the uniform policy $\mathrm{KL}(\vpi \mmid \mathrm{Unif})$ (Left: Dou Dizhu; Right: HUNL200).}
\end{figure}

\clearpage
\begin{figure}[!h]
    \centering
    \includegraphics[scale=0.8]{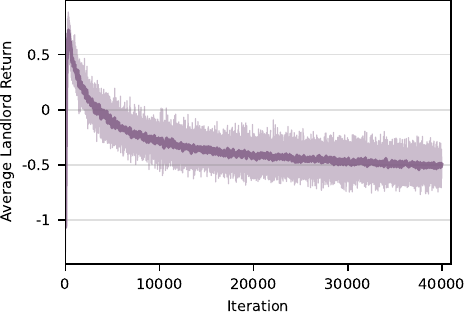}
    \hspace{8pt}
    \includegraphics[scale=0.8]{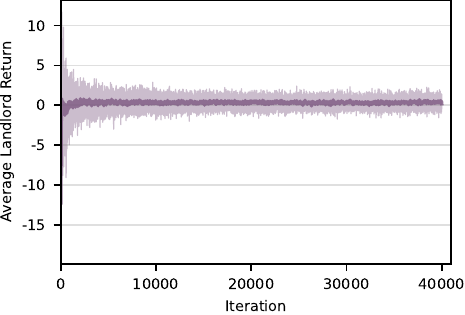}
    \caption{Average return of the first player under policy $\vpi$ (Left: Dou Dizhu; Right: HUNL200).}
\end{figure}

\begin{figure}[!h]
    \centering
    \includegraphics[scale=0.8]{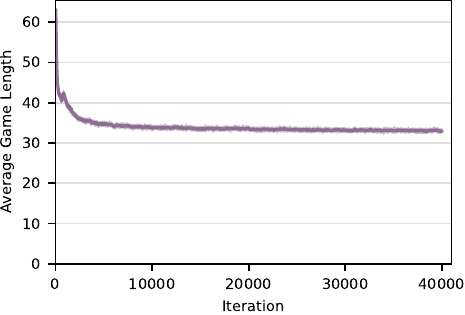}
    \hspace{8pt}
    \includegraphics[scale=0.8]{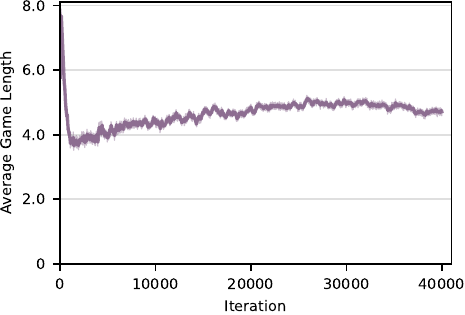}
    \caption{Average gameplay length under the current policy $\vpi$ (Left: Dou Dizhu; Right: HUNL200).}
\end{figure}

\section{Training Time}
\label{sec:training-time}

Table~\ref{tab:overhead} reports the measured per-iteration training time of VRPO and MAPPO. Since VRPO replaces the centralized $V$-critic with a centralized $Q$-critic and only changes its output layer, the overhead is $3.6$--$9.2\%$ per iteration. The critic is discarded at execution, so inference cost matches MAPPO.

\begin{table}[!h]
\centering
\small
\setlength{\tabcolsep}{6pt}
\renewcommand{\arraystretch}{1.15}
\begin{tabular}{@{}lccccc@{}}
\toprule
& APTTT & ADH3 & LiDi 2D5F & Dou Dizhu & HUNL200 \\
\midrule
MAPPO (s/iter) & $1.32$ & $1.31$ & $1.01$ & $4.65$ & $5.19$ \\
VRPO (s/iter) & $1.37$ & $1.36$ & $1.08$ & $4.95$ & $5.67$ \\
Overhead & $3.6\%$ & $3.7\%$ & $6.5\%$ & $6.5\%$ & $9.2\%$ \\
\bottomrule
\end{tabular}
\vspace{6pt}
\caption{Per-iteration training time. Overheads are computed from unrounded timings.}
\label{tab:overhead}
\vspace{-12pt}
\end{table}

\section{More Experiment Results on Dou Dizhu}
\label{sec:more-doudizhu-figures}
\begin{table}[!h]
    \vspace{-6pt}
    \centering
    \small
    \setlength{\tabcolsep}{6pt}
    \renewcommand{\arraystretch}{1.15}
    \begin{tabular}{@{}lccc@{}}
    \toprule
    \textbf{Peasants $\backslash$ Landlord} & \textbf{VRPO} & PerfectDou & DouZero \\
    \midrule
    \textbf{VRPO (ours)} & $0.707 \pm 0.011$ & $0.811 \pm 0.011$ & $0.877 \pm 0.011$ \\
    PerfectDou & $0.507  \pm 0.011$ & $0.601  \pm 0.011$ & $0.656 \pm 0.011$ \\
    DouZero  & $0.304  \pm 0.011$ & $0.373  \pm 0.011$ & $0.453  \pm 0.011$ \\
    \bottomrule
    \end{tabular}
    \vspace{6pt}
    \caption{
        Average total team gain of the row agent playing as two independent Peasants against the column playing agent as the Landlord. 
    }
    \label{tab:doudizhu-by-position}
    \vspace{-12pt}
\end{table}

\clearpage
\section{Ablation Study on General Hyperparameters}
\label{sec:ablate-hyper}

We study the effect of several PPO-style optimization hyperparameters on exact exploitability.
For each sweep, we report the mean exact exploitability over random seeds, with uncertainty shown as
$1.0 \times \mathrm{SEM}$. Lower values indicate stronger policies. We bold the best mean within each
game block.

\subsection{Trace Parameter}
We first study the trace parameter $\lambda$ used in the multi-step advantage estimation. We sweep the parameter over 
$\lambda \in \{0, 0.9, 0.95, 1.0\}$ and report the results in Table~\ref{tab:gae-lambda-sweep}.
Lower values reduce variance by using shorter traces, while larger values incorporate longer-horizon
return information. VRPO remains the best-performing algorithm across the sweep.

\begin{table}[!h]
    \vspace{-6pt}
    \centering
    \small
    \setlength{\tabcolsep}{4pt}
    \renewcommand{\arraystretch}{1.12}
    \begin{tabular}{@{}lccccc@{}}
    \toprule
    \multicolumn{6}{c}{\textbf{ADH3}} \\
    \midrule
    $\lambda$ & \textbf{VRPO} & MAPPO & IPPO & MMD & PPG \\
    \midrule
    0    & $0.080 \pm 0.002$ & $0.106 \pm 0.011$ & $0.180 \pm 0.014$ & $0.180 \pm 0.006$ & $0.154 \pm 0.009$ \\
    0.9  & \boldsymbol{$0.070 \pm 0.001$} & $0.124 \pm 0.014$ & $0.128 \pm 0.004$ & $0.132 \pm 0.011$ & $0.120 \pm 0.002$ \\
    0.95 & $0.079 \pm 0.004$ & $0.125 \pm 0.010$ & $0.138 \pm 0.012$ & $0.139 \pm 0.006$ & $0.124 \pm 0.004$ \\
    1    & $0.078 \pm 0.004$ & $0.116 \pm 0.008$ & $0.139 \pm 0.002$ & $0.144 \pm 0.003$ & $0.137 \pm 0.005$ \\
    \midrule
    \multicolumn{6}{c}{\textbf{APTTT}} \\
    \midrule
    $\lambda$ & \textbf{VRPO} & MAPPO & IPPO & MMD & PPG \\
    \midrule
    0    & \boldsymbol{$0.090 \pm 0.006$} & $0.106 \pm 0.002$ & $0.220 \pm 0.013$ & $0.225 \pm 0.014$ & $0.197 \pm 0.016$ \\
    0.9  & $0.094 \pm 0.004$ & $0.117 \pm 0.003$ & $0.126 \pm 0.005$ & $0.117 \pm 0.001$ & $0.126 \pm 0.003$ \\
    0.95 & $0.090 \pm 0.002$ & $0.117 \pm 0.002$ & $0.120 \pm 0.002$ & $0.120 \pm 0.001$ & $0.126 \pm 0.004$ \\
    1    & $0.091 \pm 0.004$ & $0.116 \pm 0.003$ & $0.126 \pm 0.005$ & $0.128 \pm 0.002$ & $0.129 \pm 0.004$ \\
    \midrule
    \multicolumn{6}{c}{\textbf{LiDi 2D5F}} \\
    \midrule
    $\lambda$ & \textbf{VRPO} & MAPPO & IPPO & MMD & PPG \\
    \midrule
    0    & $0.092 \pm 0.002$ & $0.104 \pm 0.002$ & $0.119 \pm 0.006$ & $0.141 \pm 0.006$ & $0.120 \pm 0.006$ \\
    0.9  & $0.087 \pm 0.002$ & $0.116 \pm 0.002$ & $0.108 \pm 0.004$ & $0.110 \pm 0.005$ & $0.111 \pm 0.002$ \\
    0.95 & $0.092 \pm 0.002$ & $0.117 \pm 0.003$ & $0.111 \pm 0.002$ & $0.115 \pm 0.002$ & $0.116 \pm 0.002$ \\
    1    & \boldsymbol{$0.084 \pm 0.001$} & $0.121 \pm 0.003$ & $0.112 \pm 0.002$ & $0.115 \pm 0.001$ & $0.117 \pm 0.002$ \\
    \bottomrule
    \end{tabular}
    \vspace{6pt}
    \caption{
        Exact exploitability (lower is better) under the trace parameter $\lambda$ sweep.
        Each entry reports mean $\pm$ SEM over seeds. Bold indicates the best mean in each game block.
    }
    \label{tab:gae-lambda-sweep}
    \vspace{-12pt}
\end{table}

\clearpage
\subsection{Regularization Coefficient}
We next vary the base regularization coefficient $\alpha_{\mathrm{base}}$ over $
\alpha_{\mathrm{base}} \in \{0.02, 0.05, 0.1, 0.2\}$, with results shown in
Table~\ref{tab:base-reg-coef-sweep}. This coefficient controls the strength of the regularization term
in the policy update, trading off conservative updates against the ability to move toward improved
policies.

The best VRPO setting is consistent across all three games: $0.1$ achieves the lowest exploitability. Both smaller and larger values are less reliable, indicating that this
regularization term requires a moderate setting to balance optimization stability and policy
improvement.

\begin{table}[!h]
    \vspace{-6pt}
    \centering
    \small
    \setlength{\tabcolsep}{4pt}
    \renewcommand{\arraystretch}{1.12}
    \begin{tabular}{@{}lccccc@{}}
    \toprule
    \multicolumn{6}{c}{\textbf{ADH3}} \\
    \midrule
    $\alpha_{\mathrm{base}}$ & \textbf{VRPO} & MAPPO & IPPO & MMD & PPG \\
    \midrule
    0.02 & $0.232 \pm 0.014$ & $0.134 \pm 0.008$ & $0.141 \pm 0.007$ & $0.185 \pm 0.019$ & $0.152 \pm 0.007$ \\
    0.05 & $0.102 \pm 0.005$ & $0.095 \pm 0.003$ & $0.110 \pm 0.007$ & $0.115 \pm 0.007$ & $0.110 \pm 0.010$ \\
    0.1  & \boldsymbol{$0.086 \pm 0.005$} & $0.115 \pm 0.006$ & $0.126 \pm 0.007$ & $0.146 \pm 0.004$ & $0.136 \pm 0.010$ \\
    0.2  & $0.107 \pm 0.005$ & $0.203 \pm 0.002$ & $0.226 \pm 0.002$ & $0.219 \pm 0.003$ & $0.224 \pm 0.007$ \\
    \midrule
    \multicolumn{6}{c}{\textbf{APTTT}} \\
    \midrule
    $\alpha_{\mathrm{base}}$ & \textbf{VRPO} & MAPPO & IPPO & MMD & PPG \\
    \midrule
    0.02 & $0.325 \pm 0.022$ & $0.173 \pm 0.009$ & $0.148 \pm 0.003$ & $0.142 \pm 0.011$ & $0.147 \pm 0.014$ \\
    0.05 & $0.140 \pm 0.012$ & $0.103 \pm 0.007$ & $0.105 \pm 0.006$ & $0.113 \pm 0.004$ & $0.120 \pm 0.007$ \\
    0.1  & \boldsymbol{$0.089 \pm 0.002$} & $0.122 \pm 0.003$ & $0.132 \pm 0.004$ & $0.129 \pm 0.003$ & $0.121 \pm 0.004$ \\
    0.2  & $0.107 \pm 0.004$ & $0.162 \pm 0.002$ & $0.179 \pm 0.006$ & $0.173 \pm 0.004$ & $0.169 \pm 0.003$ \\
    \midrule
    \multicolumn{6}{c}{\textbf{LiDi 2D5F}} \\
    \midrule
    $\alpha_{\mathrm{base}}$ & \textbf{VRPO} & MAPPO & IPPO & MMD & PPG \\
    \midrule
    0.02 & $0.330 \pm 0.017$ & $0.209 \pm 0.016$ & $0.173 \pm 0.009$ & $0.211 \pm 0.024$ & $0.143 \pm 0.008$ \\
    0.05 & $0.127 \pm 0.005$ & $0.135 \pm 0.002$ & $0.102 \pm 0.005$ & $0.103 \pm 0.001$ & $0.108 \pm 0.002$ \\
    0.1  & \boldsymbol{$0.099 \pm 0.004$} & $0.114 \pm 0.002$ & $0.115 \pm 0.003$ & $0.114 \pm 0.001$ & $0.111 \pm 0.003$ \\
    0.2  & $0.129 \pm 0.002$ & $0.184 \pm 0.003$ & $0.182 \pm 0.001$ & $0.179 \pm 0.003$ & $0.184 \pm 0.002$ \\
    \bottomrule
    \end{tabular}
    \vspace{6pt}
    \caption{
        Exact exploitability (lower is better) under the base regularization coefficient $\alpha_{\mathrm{base}}$ sweep. 
        Each entry reports mean $\pm$ SEM over seeds. Bold indicates the best mean in each game block.
    }
    \label{tab:base-reg-coef-sweep}
    \vspace{-12pt}
\end{table}

\clearpage
\subsection{PPO Clipping Coefficient}
We finally inspect the effect of the base clipping coefficient $\eps_{\mathrm{base}}$ used in the PPO-style policy update.
We sweep the coefficient over $\eps_{\mathrm{base}} \in \{0.02, 0.05, 0.1, 0.2\}$ and present the results in
Table~\ref{tab:base-clip-coef-sweep}. Smaller clipping coefficients impose a tighter constraint on
policy updates, while larger values allow more aggressive updates.

Across all three games, VRPO achieves its best performance at a clipping coefficient of $0.02$.
Increasing the coefficient generally degrades VRPO, especially at $0.2$, suggesting that overly loose
policy updates can hurt stability in these exact-exploitability benchmarks.

\begin{table}[!h]
    \vspace{-6pt}
    \centering
    \small
    \setlength{\tabcolsep}{4pt}
    \renewcommand{\arraystretch}{1.12}
    \begin{tabular}{@{}lccccc@{}}
    \toprule
    \multicolumn{6}{c}{\textbf{ADH3}} \\
    \midrule
    $\eps_{\mathrm{base}}$ & \textbf{VRPO} & MAPPO & IPPO & MMD & PPG \\
    \midrule
    0.02 & \boldsymbol{$0.081 \pm 0.004$} & $0.121 \pm 0.008$ & $0.123 \pm 0.004$ & $0.131 \pm 0.005$ & $0.140 \pm 0.006$ \\
    0.05 & $0.091 \pm 0.008$ & $0.125 \pm 0.002$ & $0.122 \pm 0.008$ & $0.133 \pm 0.006$ & $0.133 \pm 0.007$ \\
    0.1  & $0.125 \pm 0.012$ & $0.135 \pm 0.005$ & $0.134 \pm 0.008$ & $0.171 \pm 0.013$ & $0.138 \pm 0.013$ \\
    0.2  & $0.238 \pm 0.012$ & $0.218 \pm 0.019$ & $0.209 \pm 0.013$ & $0.233 \pm 0.012$ & $0.161 \pm 0.007$ \\
    \midrule
    \multicolumn{6}{c}{\textbf{APTTT}} \\
    \midrule
    $\eps_{\mathrm{base}}$ & \textbf{VRPO} & MAPPO & IPPO & MMD & PPG \\
    \midrule
    0.02 & \boldsymbol{$0.088 \pm 0.002$} & $0.119 \pm 0.003$ & $0.116 \pm 0.002$ & $0.128 \pm 0.005$ & $0.126 \pm 0.003$ \\
    0.05 & $0.103 \pm 0.010$ & $0.118 \pm 0.005$ & $0.116 \pm 0.004$ & $0.123 \pm 0.002$ & $0.117 \pm 0.004$ \\
    0.1  & $0.116 \pm 0.005$ & $0.129 \pm 0.003$ & $0.134 \pm 0.004$ & $0.126 \pm 0.005$ & $0.108 \pm 0.002$ \\
    0.2  & $0.205 \pm 0.001$ & $0.157 \pm 0.012$ & $0.152 \pm 0.004$ & $0.145 \pm 0.008$ & $0.108 \pm 0.007$ \\
    \midrule
    \multicolumn{6}{c}{\textbf{LiDi 2D5F}} \\
    \midrule
    $\eps_{\mathrm{base}}$ & \textbf{VRPO} & MAPPO & IPPO & MMD & PPG \\
    \midrule
    0.02 & \boldsymbol{$0.092 \pm 0.001$} & $0.117 \pm 0.002$ & $0.112 \pm 0.004$ & $0.107 \pm 0.003$ & $0.112 \pm 0.002$ \\
    0.05 & $0.098 \pm 0.003$ & $0.110 \pm 0.002$ & $0.108 \pm 0.006$ & $0.107 \pm 0.001$ & $0.103 \pm 0.003$ \\
    0.1  & $0.125 \pm 0.002$ & $0.130 \pm 0.005$ & $0.116 \pm 0.002$ & $0.124 \pm 0.006$ & $0.105 \pm 0.004$ \\
    0.2  & $0.173 \pm 0.019$ & $0.155 \pm 0.004$ & $0.138 \pm 0.006$ & $0.146 \pm 0.026$ & $0.132 \pm 0.005$ \\
    \bottomrule
    \end{tabular}
    \vspace{6pt}
    \caption{
        Exact exploitability (lower is better) under the base clipping coefficient $\eps_{\mathrm{base}}$ sweep.
        Each entry reports mean $\pm$ SEM over seeds. Bold indicates the best mean in each game block.
    }
    \label{tab:base-clip-coef-sweep}
    \vspace{-12pt}
\end{table}

\clearpage
\section{Ablation Study on VRPO-specific Hyperparameters}
\label{sec:ablate-vrpo}
We study the effect of VRPO-specific hyperparameters on the overall performance of VRPO.

\subsection{Replay Buffer Capacity}
\label{sec:ablate-vrpo-replay-buffer}

We inspect the effect of changing the replay buffer capacity $|\cD_{\mathrm{replay}}|$. We sweep the ratio between the capacity of replay buffer~$|\cD_{\mathrm{replay}}|$ and the rollout batch size $|\cD_{\mathrm{rollout}}|$ over $\{1, 4, 16, 64\}$, and present the results in Figure~\ref{fig:exact-exploitability-replay-buffer-capacity}. While a higher ratio means the critic is trained on more diverse samples, it also introduces additional lag between the stored trajectories and the current policy. When the ratio is $1$, VRPO reduces to using only the current rollout to update the critic without replay, following the standard PPO training scheme.

As shown in the figure, increasing the replay buffer capacity does not degrade performance. In our experiments, we set the replay buffer capacity to $64 \times |\cD_{\mathrm{rollout}}|$ to balance sample divergence and total memory consumption.

\begin{figure*}[!h]
    \centering
    \includegraphics[scale=0.8]{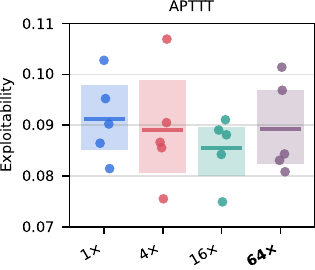}
    \hspace{8pt}
    \includegraphics[scale=0.8]{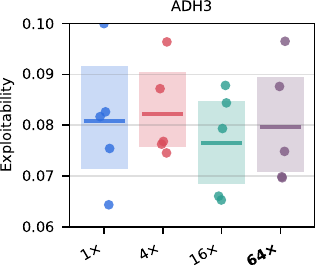}
    \hspace{8pt}
    \includegraphics[scale=0.8]{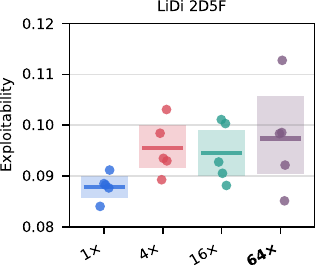}
    \caption{Exact exploitability (lower is better) of agents in APTTT (Abrupt Phantom Tic-Tac-Toe), ADH3 (Abrupt Dark Hex 3), and LiDi 2D5F (Liar's Dice with 2 Dice and 5 Faces), trained with VRPO under different replay buffer capacity. }
    \label{fig:exact-exploitability-replay-buffer-capacity}
\end{figure*}

\subsection{Number of Critic Epochs}

We inspect the effect of changing the number of critic epochs $K_{\mathrm{critic}}$ per iteration. We sweep $K_{\mathrm{critic}} \in \{1, 2, 4, 8\}$ and present the results in Figure~\ref{fig:exact-exploitability-aux-epoch}. As shown in the figure, a larger number of critic epochs leads to a more accurate critic, resulting in more stable performance, as expected. In our main experiments, we set $K_{\mathrm{critic}} = 4$ to balance stability and computational cost, matching the PPO baselines used in CleanRL~\citep{huang2022cleanrl}.

\begin{figure*}[!h]
    \centering
    \includegraphics[scale=0.8]{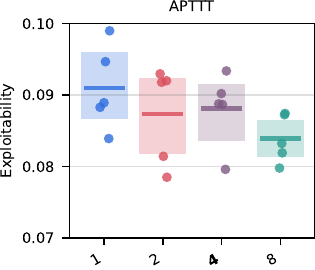}
    \hspace{8pt}
    \includegraphics[scale=0.8]{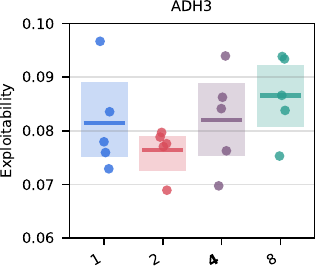}
    \hspace{8pt}
    \includegraphics[scale=0.8]{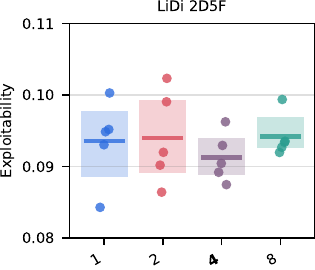}
    \caption{Exact exploitability (lower is better) of agents in APTTT (Abrupt Phantom Tic-Tac-Toe), ADH3 (Abrupt Dark Hex 3), and LiDi 2D5F (Liar's Dice with 2 Dice and 5 Faces), trained with VRPO under different numbers of critic epochs.}
    \label{fig:exact-exploitability-aux-epoch}
\end{figure*}

\newpage
\clearpage
\section*{NeurIPS Paper Checklist}

\begin{enumerate}

\item {\bf Claims}
    \item[] Question: Do the main claims made in the abstract and introduction accurately reflect the paper's contributions and scope?
    \item[] Answer: \answerYes{}%
    \item[] Justification: The claims made in the abstract and introduction accurately reflect the paper's contributions and scope.
    \item[] Guidelines:
    \begin{itemize}
        \item The answer \answerNA{} means that the abstract and introduction do not include the claims made in the paper.
        \item The abstract and/or introduction should clearly state the claims made, including the contributions made in the paper and important assumptions and limitations. A \answerNo{} or \answerNA{} answer to this question will not be perceived well by the reviewers. 
        \item The claims made should match theoretical and experimental results, and reflect how much the results can be expected to generalize to other settings. 
        \item It is fine to include aspirational goals as motivation as long as it is clear that these goals are not attained by the paper. 
    \end{itemize}

\item {\bf Limitations}
    \item[] Question: Does the paper discuss the limitations of the work performed by the authors?
    \item[] Answer: \answerYes{}%
    \item[] Justification: We have a separate paragraph that discusses the limitations.
    \item[] Guidelines:
    \begin{itemize}
        \item The answer \answerNA{} means that the paper has no limitation while the answer \answerNo{} means that the paper has limitations, but those are not discussed in the paper. 
        \item The authors are encouraged to create a separate ``Limitations'' section in their paper.
        \item The paper should point out any strong assumptions and how robust the results are to violations of these assumptions (e.g., independence assumptions, noiseless settings, model well-specification, asymptotic approximations only holding locally). The authors should reflect on how these assumptions might be violated in practice and what the implications would be.
        \item The authors should reflect on the scope of the claims made, e.g., if the approach was only tested on a few datasets or with a few runs. In general, empirical results often depend on implicit assumptions, which should be articulated.
        \item The authors should reflect on the factors that influence the performance of the approach. For example, a facial recognition algorithm may perform poorly when image resolution is low or images are taken in low lighting. Or a speech-to-text system might not be used reliably to provide closed captions for online lectures because it fails to handle technical jargon.
        \item The authors should discuss the computational efficiency of the proposed algorithms and how they scale with dataset size.
        \item If applicable, the authors should discuss possible limitations of their approach to address problems of privacy and fairness.
        \item While the authors might fear that complete honesty about limitations might be used by reviewers as grounds for rejection, a worse outcome might be that reviewers discover limitations that aren't acknowledged in the paper. The authors should use their best judgment and recognize that individual actions in favor of transparency play an important role in developing norms that preserve the integrity of the community. Reviewers will be specifically instructed to not penalize honesty concerning limitations.
    \end{itemize}

\item {\bf Theory assumptions and proofs}
    \item[] Question: For each theoretical result, does the paper provide the full set of assumptions and a complete (and correct) proof?
    \item[] Answer: \answerYes{}%
    \item[] Justification: The theoretical result is proved in the appendix.
    \item[] Guidelines:
    \begin{itemize}
        \item The answer \answerNA{} means that the paper does not include theoretical results. 
        \item All the theorems, formulas, and proofs in the paper should be numbered and cross-referenced.
        \item All assumptions should be clearly stated or referenced in the statement of any theorems.
        \item The proofs can either appear in the main paper or the supplemental material, but if they appear in the supplemental material, the authors are encouraged to provide a short proof sketch to provide intuition. 
        \item Inversely, any informal proof provided in the core of the paper should be complemented by formal proofs provided in appendix or supplemental material.
        \item Theorems and Lemmas that the proof relies upon should be properly referenced. 
    \end{itemize}

    \item {\bf Experimental result reproducibility}
    \item[] Question: Does the paper fully disclose all the information needed to reproduce the main experimental results of the paper to the extent that it affects the main claims and/or conclusions of the paper (regardless of whether the code and data are provided or not)?
    \item[] Answer: \answerYes{}%
    \item[] Justification: We provide the detailed information in the appendix.
    \item[] Guidelines:
    \begin{itemize}
        \item The answer \answerNA{} means that the paper does not include experiments.
        \item If the paper includes experiments, a \answerNo{} answer to this question will not be perceived well by the reviewers: Making the paper reproducible is important, regardless of whether the code and data are provided or not.
        \item If the contribution is a dataset and\slash or model, the authors should describe the steps taken to make their results reproducible or verifiable. 
        \item Depending on the contribution, reproducibility can be accomplished in various ways. For example, if the contribution is a novel architecture, describing the architecture fully might suffice, or if the contribution is a specific model and empirical evaluation, it may be necessary to either make it possible for others to replicate the model with the same dataset, or provide access to the model. In general. releasing code and data is often one good way to accomplish this, but reproducibility can also be provided via detailed instructions for how to replicate the results, access to a hosted model (e.g., in the case of a large language model), releasing of a model checkpoint, or other means that are appropriate to the research performed.
        \item While NeurIPS does not require releasing code, the conference does require all submissions to provide some reasonable avenue for reproducibility, which may depend on the nature of the contribution. For example
        \begin{enumerate}
            \item If the contribution is primarily a new algorithm, the paper should make it clear how to reproduce that algorithm.
            \item If the contribution is primarily a new model architecture, the paper should describe the architecture clearly and fully.
            \item If the contribution is a new model (e.g., a large language model), then there should either be a way to access this model for reproducing the results or a way to reproduce the model (e.g., with an open-source dataset or instructions for how to construct the dataset).
            \item We recognize that reproducibility may be tricky in some cases, in which case authors are welcome to describe the particular way they provide for reproducibility. In the case of closed-source models, it may be that access to the model is limited in some way (e.g., to registered users), but it should be possible for other researchers to have some path to reproducing or verifying the results.
        \end{enumerate}
    \end{itemize}

\item {\bf Open access to data and code}
    \item[] Question: Does the paper provide open access to the data and code, with sufficient instructions to faithfully reproduce the main experimental results, as described in supplemental material?
    \item[] Answer: \answerYes{}%
    \item[] Justification: We provide the codebase needed to reproduce the results.
    \item[] Guidelines:
    \begin{itemize}
        \item The answer \answerNA{} means that paper does not include experiments requiring code.
        \item Please see the NeurIPS code and data submission guidelines (\url{https://neurips.cc/public/guides/CodeSubmissionPolicy}) for more details.
        \item While we encourage the release of code and data, we understand that this might not be possible, so \answerNo{} is an acceptable answer. Papers cannot be rejected simply for not including code, unless this is central to the contribution (e.g., for a new open-source benchmark).
        \item The instructions should contain the exact command and environment needed to run to reproduce the results. See the NeurIPS code and data submission guidelines (\url{https://neurips.cc/public/guides/CodeSubmissionPolicy}) for more details.
        \item The authors should provide instructions on data access and preparation, including how to access the raw data, preprocessed data, intermediate data, and generated data, etc.
        \item The authors should provide scripts to reproduce all experimental results for the new proposed method and baselines. If only a subset of experiments are reproducible, they should state which ones are omitted from the script and why.
        \item At submission time, to preserve anonymity, the authors should release anonymized versions (if applicable).
        \item Providing as much information as possible in supplemental material (appended to the paper) is recommended, but including URLs to data and code is permitted.
    \end{itemize}

\item {\bf Experimental setting/details}
    \item[] Question: Does the paper specify all the training and test details (e.g., data splits, hyperparameters, how they were chosen, type of optimizer) necessary to understand the results?
    \item[] Answer: \answerYes{}%
    \item[] Justification: We specify all the training and test details.
    \item[] Guidelines:
    \begin{itemize}
        \item The answer \answerNA{} means that the paper does not include experiments.
        \item The experimental setting should be presented in the core of the paper to a level of detail that is necessary to appreciate the results and make sense of them.
        \item The full details can be provided either with the code, in appendix, or as supplemental material.
    \end{itemize}

\item {\bf Experiment statistical significance}
    \item[] Question: Does the paper report error bars suitably and correctly defined or other appropriate information about the statistical significance of the experiments?
    \item[] Answer: \answerYes{}%
    \item[] Justification: We report error bars about the statistical significance of the experiments.
    \item[] Guidelines:
    \begin{itemize}
        \item The answer \answerNA{} means that the paper does not include experiments.
        \item The authors should answer \answerYes{} if the results are accompanied by error bars, confidence intervals, or statistical significance tests, at least for the experiments that support the main claims of the paper.
        \item The factors of variability that the error bars are capturing should be clearly stated (for example, train/test split, initialization, random drawing of some parameter, or overall run with given experimental conditions).
        \item The method for calculating the error bars should be explained (closed form formula, call to a library function, bootstrap, etc.)
        \item The assumptions made should be given (e.g., Normally distributed errors).
        \item It should be clear whether the error bar is the standard deviation or the standard error of the mean.
        \item It is OK to report 1-sigma error bars, but one should state it. The authors should preferably report a 2-sigma error bar than state that they have a 96\% CI, if the hypothesis of Normality of errors is not verified.
        \item For asymmetric distributions, the authors should be careful not to show in tables or figures symmetric error bars that would yield results that are out of range (e.g., negative error rates).
        \item If error bars are reported in tables or plots, the authors should explain in the text how they were calculated and reference the corresponding figures or tables in the text.
    \end{itemize}

\item {\bf Experiments compute resources}
    \item[] Question: For each experiment, does the paper provide sufficient information on the computer resources (type of compute workers, memory, time of execution) needed to reproduce the experiments?
    \item[] Answer: \answerYes{}%
    \item[] Justification: We provide sufficient information on the computer resources.
    \item[] Guidelines:
    \begin{itemize}
        \item The answer \answerNA{} means that the paper does not include experiments.
        \item The paper should indicate the type of compute workers CPU or GPU, internal cluster, or cloud provider, including relevant memory and storage.
        \item The paper should provide the amount of compute required for each of the individual experimental runs as well as estimate the total compute. 
        \item The paper should disclose whether the full research project required more compute than the experiments reported in the paper (e.g., preliminary or failed experiments that didn't make it into the paper). 
    \end{itemize}
    
\item {\bf Code of ethics}
    \item[] Question: Does the research conducted in the paper conform, in every respect, with the NeurIPS Code of Ethics \url{https://neurips.cc/public/EthicsGuidelines}?
    \item[] Answer: \answerYes{}%
    \item[] Justification: Yes
    \item[] Guidelines:
    \begin{itemize}
        \item The answer \answerNA{} means that the authors have not reviewed the NeurIPS Code of Ethics.
        \item If the authors answer \answerNo, they should explain the special circumstances that require a deviation from the Code of Ethics.
        \item The authors should make sure to preserve anonymity (e.g., if there is a special consideration due to laws or regulations in their jurisdiction).
    \end{itemize}

\item {\bf Broader impacts}
    \item[] Question: Does the paper discuss both potential positive societal impacts and negative societal impacts of the work performed?
    \item[] Answer: \answerYes{}%
    \item[] Justification: This work studies self-play RL in games. Potential negative uses include stronger automated agents for online gambling or unfair play.
    \item[] Guidelines:
    \begin{itemize}
        \item The answer \answerNA{} means that there is no societal impact of the work performed.
        \item If the authors answer \answerNA{} or \answerNo, they should explain why their work has no societal impact or why the paper does not address societal impact.
        \item Examples of negative societal impacts include potential malicious or unintended uses (e.g., disinformation, generating fake profiles, surveillance), fairness considerations (e.g., deployment of technologies that could make decisions that unfairly impact specific groups), privacy considerations, and security considerations.
        \item The conference expects that many papers will be foundational research and not tied to particular applications, let alone deployments. However, if there is a direct path to any negative applications, the authors should point it out. For example, it is legitimate to point out that an improvement in the quality of generative models could be used to generate Deepfakes for disinformation. On the other hand, it is not needed to point out that a generic algorithm for optimizing neural networks could enable people to train models that generate Deepfakes faster.
        \item The authors should consider possible harms that could arise when the technology is being used as intended and functioning correctly, harms that could arise when the technology is being used as intended but gives incorrect results, and harms following from (intentional or unintentional) misuse of the technology.
        \item If there are negative societal impacts, the authors could also discuss possible mitigation strategies (e.g., gated release of models, providing defenses in addition to attacks, mechanisms for monitoring misuse, mechanisms to monitor how a system learns from feedback over time, improving the efficiency and accessibility of ML).
    \end{itemize}
    
\item {\bf Safeguards}
    \item[] Question: Does the paper describe safeguards that have been put in place for responsible release of data or models that have a high risk for misuse (e.g., pre-trained language models, image generators, or scraped datasets)?
    \item[] Answer: \answerNA{}%
    \item[] Justification: No such risks.
    \item[] Guidelines:
    \begin{itemize}
        \item The answer \answerNA{} means that the paper poses no such risks.
        \item Released models that have a high risk for misuse or dual-use should be released with necessary safeguards to allow for controlled use of the model, for example by requiring that users adhere to usage guidelines or restrictions to access the model or implementing safety filters. 
        \item Datasets that have been scraped from the Internet could pose safety risks. The authors should describe how they avoided releasing unsafe images.
        \item We recognize that providing effective safeguards is challenging, and many papers do not require this, but we encourage authors to take this into account and make a best faith effort.
    \end{itemize}

\item {\bf Licenses for existing assets}
    \item[] Question: Are the creators or original owners of assets (e.g., code, data, models), used in the paper, properly credited and are the license and terms of use explicitly mentioned and properly respected?
    \item[] Answer: \answerYes{}%
    \item[] Justification: We credit the license and terms for the used assets.
    \item[] Guidelines:
    \begin{itemize}
        \item The answer \answerNA{} means that the paper does not use existing assets.
        \item The authors should cite the original paper that produced the code package or dataset.
        \item The authors should state which version of the asset is used and, if possible, include a URL.
        \item The name of the license (e.g., CC-BY 4.0) should be included for each asset.
        \item For scraped data from a particular source (e.g., website), the copyright and terms of service of that source should be provided.
        \item If assets are released, the license, copyright information, and terms of use in the package should be provided. For popular datasets, \url{paperswithcode.com/datasets} has curated licenses for some datasets. Their licensing guide can help determine the license of a dataset.
        \item For existing datasets that are re-packaged, both the original license and the license of the derived asset (if it has changed) should be provided.
        \item If this information is not available online, the authors are encouraged to reach out to the asset's creators.
    \end{itemize}

\item {\bf New assets}
    \item[] Question: Are new assets introduced in the paper well documented and is the documentation provided alongside the assets?
    \item[] Answer: \answerNA{}%
    \item[] Justification: The paper does not release new assets.
    \item[] Guidelines:
    \begin{itemize}
        \item The answer \answerNA{} means that the paper does not release new assets.
        \item Researchers should communicate the details of the dataset\slash code\slash model as part of their submissions via structured templates. This includes details about training, license, limitations, etc. 
        \item The paper should discuss whether and how consent was obtained from people whose asset is used.
        \item At submission time, remember to anonymize your assets (if applicable). You can either create an anonymized URL or include an anonymized zip file.
    \end{itemize}

\item {\bf Crowdsourcing and research with human subjects}
    \item[] Question: For crowdsourcing experiments and research with human subjects, does the paper include the full text of instructions given to participants and screenshots, if applicable, as well as details about compensation (if any)? 
    \item[] Answer: \answerNA{}%
    \item[] Justification: The paper does not involve crowdsourcing nor research with human subjects.
    \item[] Guidelines:
    \begin{itemize}
        \item The answer \answerNA{} means that the paper does not involve crowdsourcing nor research with human subjects.
        \item Including this information in the supplemental material is fine, but if the main contribution of the paper involves human subjects, then as much detail as possible should be included in the main paper. 
        \item According to the NeurIPS Code of Ethics, workers involved in data collection, curation, or other labor should be paid at least the minimum wage in the country of the data collector. 
    \end{itemize}

\item {\bf Institutional review board (IRB) approvals or equivalent for research with human subjects}
    \item[] Question: Does the paper describe potential risks incurred by study participants, whether such risks were disclosed to the subjects, and whether Institutional Review Board (IRB) approvals (or an equivalent approval/review based on the requirements of your country or institution) were obtained?
    \item[] Answer: \answerNA{}%
    \item[] Justification: The paper does not involve crowdsourcing nor research with human subjects.
    \item[] Guidelines:
    \begin{itemize}
        \item The answer \answerNA{} means that the paper does not involve crowdsourcing nor research with human subjects.
        \item Depending on the country in which research is conducted, IRB approval (or equivalent) may be required for any human subjects research. If you obtained IRB approval, you should clearly state this in the paper. 
        \item We recognize that the procedures for this may vary significantly between institutions and locations, and we expect authors to adhere to the NeurIPS Code of Ethics and the guidelines for their institution. 
        \item For initial submissions, do not include any information that would break anonymity (if applicable), such as the institution conducting the review.
    \end{itemize}

\item {\bf Declaration of LLM usage}
    \item[] Question: Does the paper describe the usage of LLMs if it is an important, original, or non-standard component of the core methods in this research? Note that if the LLM is used only for writing, editing, or formatting purposes and does \emph{not} impact the core methodology, scientific rigor, or originality of the research, declaration is not required.
    \item[] Answer: \answerNA{}%
    \item[] Justification: The core method development in this research does not involve LLMs as any important, original, or non-standard components.
    \item[] Guidelines:
    \begin{itemize}
        \item The answer \answerNA{} means that the core method development in this research does not involve LLMs as any important, original, or non-standard components.
        \item Please refer to our LLM policy in the NeurIPS handbook for what should or should not be described.
    \end{itemize}

\end{enumerate}

\end{document}